%% file: arxiv.tex
\pdfoutput=1
\documentclass{article}
\usepackage{times,authblk}

\usepackage{fullpage}

\usepackage{geometry}
\usepackage{hyperref,url}
\usepackage[numbers]{natbib}
\usepackage{algorithm,algorithmicx}
\usepackage{algpseudocode}
\usepackage{subfigure,multirow,graphicx,color,rotating}
\usepackage{xspace}
\usepackage{tikz}
\usepackage{amsthm,amsmath,amssymb,amsfonts}
\usepackage{enumerate}
\usepackage{verbatim,bbold}
\usepackage{mkolar_definitions}

\input{notation.tex}

\newcount\Comments
\definecolor{teal}{rgb}{0.3,0.8,0.8}
\definecolor{darkgreen}{rgb}{0,0.5,0}
\definecolor{darkred}{rgb}{0.7,0,0}
\definecolor{teal}{rgb}{0.3,0.8,0.8}
\definecolor{blue}{rgb}{0,0,1}
\definecolor{purple}{rgb}{0.5,0,1}
\newcommand{\kibitz}[2]{\ifnum\Comments=1\textcolor{#1}{#2}\fi}

\newcommand{\order}{\ensuremath{\mathcal{O}}}
\newcommand{\otil}{\ensuremath{\tilde{\order}}}

\newcommand{\version}{arxiv}
\ifthenelse{\equal{\version}{arxiv}}{
\renewcommand{\cite}{\citep}

}{

}

\newenvironment{packed_enum}{
\begin{enumerate}
\setlength{\itemsep}{1pt}
\setlength{\parskip}{0pt}
\setlength{\parsep}{0pt}
}{\end{enumerate}}

\newcounter{qcounter}
\newenvironment{enums}
 {\begin{list}{\arabic{qcounter}.}
 {\usecounter{qcounter} \setlength{\topsep}{0in} \setlength{\partopsep}{0in}
  \setlength{\itemsep}{\parskip}
  \setlength{\parsep}{0in} 
  \setlength{\leftmargin}{0.21in} \setlength{\rightmargin}{0in}
  \setlength{\listparindent}{0.0in} \setlength{\labelwidth}{0.07in}
  \setlength{\labelsep}{0.1in} \setlength{\itemindent}{-0.04in}}}
 {\end{list}}

\title{PAC Reinforcement Learning with Rich Observations}

\author[1]{Akshay Krishnamurthy
\thanks{akshay@cs.umass.edu}}
\author[2]{Alekh Agarwal
\thanks{alekha@microsoft.com}}
\author[2]{John Langford
\thanks{jcl@microsoft.com}}
\affil{University of Massachusetts, Amherst\\
Amherst, MA 01003}
\affil[2]{Microsoft Research\\
New York, NY 10011}

\begin{document}
\maketitle

\input{abstract2.tex}
\input{intro.tex}
\input{model.tex}

\input{related.tex}
\input{approach.tex}
\input{discussion.tex}

\section*{Acknowledgements}
We thank Akshay Balsubramani and Hal Daum\'{e} III for formative
discussions, and we thank Tzu-Kuo Huang and Nan Jiang for carefully
reading an early draft of this paper.  This work was carried out while
AK was at Microsoft Research.

\appendix
\input{appendix.tex}

\bibliography{rl}
\bibliographystyle{plainnat}

\end{document}

%% file: notation.tex
\renewcommand{\models}{Contextual-MDPs\xspace}

\newcommand{\Xspace}{\Xcal}
\newcommand{\Aspace}{\Acal}
\newcommand{\Asize}{K}
\newcommand{\Sspace}{\Scal}
\newcommand{\Ssize}{M}
\newcommand{\Fspace}{\Fcal}
\newcommand{\Fsize}{N}

\newcommand{\Trans}{\Gamma}

\newcommand{\ntest}{n_{\textrm{test}}}
\newcommand{\ntrain}{n_{\textrm{train}}}
\newcommand{\ndemand}{n_{\textrm{1}}}
\newcommand{\ndemandtwo}{n_{\textrm{2}}}

\newcommand{\epstest}{\epsilon_{\textrm{test}}}

\newcommand{\epsdemand}{\epsilon_{\textrm{demand}}}

\newcommand{\vfunc}[1]{V^{#1}}
\newcommand{\vhat}[1]{\hat{V}^{#1}}
\newcommand{\fgpvfunc}[3]{\vfunc{#1}(#2,\pi_{#3})}
\newcommand{\fgpvhat}[3]{\vhat{#1}(#2,\pi_{#3})}
\newcommand{\fpvfunc}[2]{\fgpvfunc{#1}{#2}{#1}}
\newcommand{\fpvhat}[2]{\fgpvhat{#1}{#2}{#1}}
\newcommand{\fpvshort}[2]{\vfunc{#1}(#2)}
\newcommand{\fpvhatshort}[2]{\vhat{#1}(#2)}
\newcommand{\favshort}[2]{\vfunc{#1}_{#2}}
\newcommand{\favhatshort}[2]{\vhat{#1}_{#2}}

\newcommand{\true}{\texttt{true}\xspace}
\newcommand{\false}{\texttt{false}\xspace}
\newcommand{\statelearned}{\textsc{Consensus}\xspace}

\newcommand{\dfslearn}{\textsc{DFS-Learn}\xspace}
\newcommand{\regelim}{\textsc{TD-Elim}\xspace}
\newcommand{\ondemand}{\textsc{Explore-on-Demand}\xspace}

\newcommand{\alg}{\textsc{LSVEE}\xspace}
\newcommand{\alglong}{{Least Squares Value Elimination by Exploration}\xspace}

\newcommand{\esterr}{\phi}
\newcommand{\tdbias}{\tau_1}
\newcommand{\tdrisk}{\psi}
\newcommand{\slbias}{\tau_2}
\newcommand{\treeroot}{\varnothing}

\newcommand{\vs}{\ensuremath{V^\star}}
\newcommand{\qs}{\ensuremath{Q^\star}}

\newcommand{\dist}[2]{\ensuremath{{D_{#1}({#2})}}}
\newcommand{\dst}[1]{\ensuremath{{D_{#1}}}}

%% file: abstract2.tex
\begin{abstract}
We propose and study a new model for reinforcement learning with rich
observations, generalizing contextual bandits to sequential decision
making.  These models require an agent to take actions based on
observations (features) with the goal of achieving long-term
performance competitive with a large set of policies.  To avoid
barriers to sample-efficient learning associated with large
observation spaces and general POMDPs, we focus on problems that can
be summarized by a small number of hidden states and have long-term
rewards that are predictable by a reactive function class.  In this
setting, we design and analyze a new reinforcement learning algorithm,
\alglong. We prove that the algorithm learns near optimal behavior
after a number of episodes that is polynomial in all relevant
parameters, logarithmic in the number of policies, and independent of
the size of the observation space. Our result provides theoretical
justification for reinforcement learning with function approximation.
\end{abstract}

%% file: intro.tex
\section{Introduction}
\label{sec:intro}
The Atari Reinforcement Learning research program~\cite{mnih2015human}
has highlighted a critical deficiency of practical reinforcement
learning algorithms in settings with rich observation spaces: they
cannot effectively solve problems that require sophisticated exploration.
How can we construct Reinforcement Learning (RL) algorithms which
effectively plan and plan to explore?

In RL theory, this is a solved problem for Markov
Decision Processes (MDPs)~\cite{kearns2002near, brafman2003r,
  strehl2006pac}.  Why do these results not apply?

An easy response is, ``because the hard games are not MDPs.''  This
may be true for some of the hard games, but it is
misleading---popular algorithms like $Q$-learning with
$\epsilon$-greedy exploration do not even engage in minimal planning
and global exploration\footnote{We use ``global exploration'' to
  distinguish the sophisticated exploration strategies required to solve
  an MDP efficiently from exponentially less efficient alternatives
  such as $\epsilon$-greedy.} as is required to solve MDPs
efficiently.  MDP-optimized global exploration has also been avoided
because of a polynomial dependence on the number of unique
observations which is intractably large with observations from a
visual sensor.

In contrast, supervised and contextual bandit learning algorithms have
\emph{no} dependence on the number of observations and at most a
logarithmic dependence on the size of the underlying policy set.
Approaches to RL with a weak dependence on these quantities
exist~\cite{kearns2002sparse} but suffer from an exponential
dependence on the time horizon---with $\Asize$ actions and a horizon
of $H$, they require $\Omega(\Asize^H)$ samples.  Examples show that
this dependence is necessary, although they typically require a large
number of states.  Can we find an RL algorithm with no dependence on
the number of unique observations and a polynomial dependence on the
number of actions $\Asize$, the number of necessary states $\Ssize$,
the horizon $H$, and the policy complexity $\log(|\Pi|)$?

To begin answering this question we consider a simplified setting with
episodes of bounded length $H$ and deterministic state transitions. We
further assume that we have a function class that contains the optimal
observation-action value function $Q^\star$. These simplifications make the
problem significantly more tractable without trivializing the core
goal of designing a $\textrm{Poly}(\Asize,\Ssize,H,\log(|\Pi|)))$
algorithm.
To this end, our contributions are:
\begin{enums}
\item A new class of models for studying reinforcement learning with
  rich observations.  These models generalize both contextual bandits
  and small-state MDPs, but do not exhibit the partial observability
  issues of more complex models like POMDPs.  We show
  \emph{exponential lower bounds} on sample complexity in the absence
  of the assumptions to justify our model.
\item A new reinforcement learning algorithm \alglong (\alg) and a PAC
  guarantee that it finds a policy that is at most $\epsilon$
  sub-optimal (with the above assumptions) using
  $\order\left(\frac{\Ssize \Asize^2 H^6}{\epsilon^3}
  \log(|\Pi|)\right)$ samples, with no dependence on the number of
  unique observations. This is done by combining ideas from contextual
  bandits with a novel state equality test and a global exploration
  technique.
Like initial contextual bandit
approaches~\cite{agarwal2012contextual}, the algorithm is
computationally inefficient since it requires enumeration of the
policy class, an aspect we hope to address in future work.
\end{enums}
\alg uses a function class to approximate future rewards, and
thus lends theoretical backing for reinforcement learning with
function approximation, which is the empirical state-of-the-art.

%% file: model.tex
\section{The Model}
\label{sec:model}
Our model is a \textbf{Contextual Decision Process}, a term we use
broadly to refer to any sequential decision making task where an agent
must make decision on the basis of rich features (context) to optimize
long-term reward.  In this section, we introduce the model, starting
with basic notation. Let $H \in \NN$ denote an episode length,
$\Xspace \subseteq \mathbb{R}^d$ an observation space, $\Aspace$ a
finite set of actions, and $\Sspace$ a finite set of latent states.
Let $\Asize \triangleq |\Aspace|$.  We partition $\Sspace$ into $H$
disjoint groups $\Sspace_1, \ldots, \Sspace_H$, each of size at most
$\Ssize$.
For a set $P$,
$\Delta(P)$ denotes the set of distributions over $P$.


\subsection{Basic Definitions}
\label{sec:model_definition}
Our model is defined by the tuple $(\Trans_1, \Trans, D)$ where
$\Trans_1 \in \Delta(\Sspace_1)$ denotes a starting state
distribution, $\Trans: (\Sspace \times \Aspace) \rightarrow
\Delta(\Sspace)$ denotes the transition dynamics, and $\dst{s} \in
\Delta(\Xspace \times [0,1]^{\Asize})$ associates a distribution over
observation-reward pairs with each state $s \in \Scal$.  We also use
$\dst{s}$ to denote the marginal distribution over observations (usage
will be clear from context) and use $\dst{s|x}$ for the conditional
distribution over reward given the observation $x$ in state $s$.  The
marginal and conditional probabilities are referred to as
$\dist{s}{x}$ and $\dist{s|x}{r}$.

We assume that the process is \emph{layered} (also known as loop-free
or acyclic) so that for any $s_h \in \Sspace_h$
and action $a \in \Aspace$, $\Trans(s_h,a) \in
\Delta(\Scal_{h+1})$.  Thus, 
the environment transitions from state space $\Scal_1$ up to
$\Scal_H$ via a sequence of actions.  Layered structure allows us to
avoid indexing policies and $Q$-functions with time, which enables
concise notation.

Each episode produces a full record of interaction $(s_1, x_1, a_1,
r_1, \ldots, s_H,x_H,a_H,r_H)$ where $s_1 \sim \Trans_1$, $s_h \sim
\Trans(s_{h-1}, a_{h-1})$, $(x_h,r_h) \sim \dst{s_h}$ and all actions
$a_h$ are chosen by the learning agent.  The record of interaction
observed by the learner is $(x_1, a_1, r_1(a_1), \ldots,
x_H,a_H,r_H(a_H))$ and at time point $h$, the learner may use all
observable information up to and including $x_h$ to select $a_h$.  Notice that
all state information and rewards for alternative actions are
unobserved by the learning agent.

The learner's reward for an episode is $\sum_{h=1}^Hr_h(a_h)$, and the
goal is to maximize the expected cumulative reward, $R =
\EE[\sum_{h=1}^Hr_h(a_h)]$, where the expectation accounts for all the
randomness in the model and the learner.  We assume that almost surely
$\sum_{h=1}^H r_h(a_h) \in [0,1]$ for any action sequence.



In this model, the optimal expected reward achievable can be computed
recursively as
\begin{align}
\vs \triangleq \EE_{s\sim \Trans_1}[\vs(s)] \quad \mbox{with} \quad
\vs(s) \triangleq \EE_{x\sim \dst{s}} \max_a
\EE_{r\sim\dst{s|x}}\left[ r(a)+ \EE_{s'\sim
    \Gamma(s,a)}\vs(s')\right]. \label{eq:vstar}
\end{align}
As the base case, we assume that for states $s \in \Scal_H$, all
actions transition to a terminal state $s_{H+1}$ with $\vs(s_{H+1}) \triangleq
0$. For each $(s,x)$ pair such that $\dist{s}{x} > 0$ we also define a
$\qs$ function as
\begin{align}
  \qs_s(x,a) &\triangleq \EE_{r\sim\dst{s|x}} \left[r(a) + \EE_{s'\sim
      \Gamma(s,a)} \vs(s') \right]. 
  \label{eqn:qstar}
\end{align}
This function captures the optimal choice of action given this (state,
observation) pair and therefore encodes optimal behavior in the model.

With no further assumptions, the above model is a \emph{layered
  episodic Partially Observable Markov Decision Process} (LE-POMDP).
Both learning and planning are notoriously challenging in POMDPs,
because the optimal policy depends on the entire trajectory and the
complexity of learning such a policy grows exponentially with $H$ (see
e.g.~\citet{kearns2002sparse} as well as
Propositions~\ref{prop:agnostic_lb} and~\ref{prop:realizable_lb}
below).  Our model avoids this statistical barrier with two
assumptions: (a) we consider only reactive policies, and (b) we assume
access to a class of functions that can realize the $\qs$ function.
Both assumptions are implicit in the empirical state of the art RL
results.  They also eliminate issues related to partial observability,
allowing us to focus on our core goal of systematic exploration.  We
describe both assumptions in detail before formally defining the
model.

\textbf{Reactive Policies:} One approach taken by some prior
theoretical work is to consider \emph{reactive} (or memoryless)
policies that use only the current observation to select an
action~\cite{meuleau1999learning,azizzadenesheli2016reinforcement}.
Memorylessness is slightly generalized in the recent empirical
advances in RL, which typically employ policies that depend only on
the few most recent observations~\cite{mnih2015human}.

A reactive policy $\pi: \Xspace \rightarrow \Aspace$ is a strategy for
navigating the search space by taking actions $\pi(x)$ given
observation $x$.  The expected reward for a policy is defined recursively through
\begin{align*}
V(\pi) \triangleq \EE_{s \sim \Trans_1}[V(s,\pi)] \quad \mbox{and} \quad V(s,\pi) \triangleq \EE_{(x,r) \sim \dst{s}}\left[r(\pi(x)) + \EE_{s'\sim \Trans(s,\pi(x))}V(s',\pi)\right].
\end{align*}

A natural learning goal is to identify a policy with maximal value
$V(\pi)$ from a given collection of reactive policies $\Pi$.
Unfortunately, even when restricting to reactive policies, learning in
POMDPs requires exponentially many samples, as we show in the next
lower bound.
\begin{proposition}
\label{prop:agnostic_lb}
Fix $H,K\in \NN$ with $K \ge 2$ and $\epsilon \in (0,\sqrt{1/8})$.
For any algorithm, there exists a LE-POMDP with horizon $H$, $K$
actions, and $2H$ total states; a class $\Pi$ of reactive policies
with $|\Pi| = K^H$; and a constant $c > 0$ such that the probability
that the algorithm outputs a policy $\hat{\pi}$ with $V(\hat{\pi}) >
\max_{\pi \in \Pi}V(\pi) - \epsilon$ after collecting $T$ trajectories
is at most $2/3$ for all $T \le cK^H/\epsilon^2$.
\end{proposition}
This lower bound precludes a $\textrm{Poly}(K,M,H,\log(|\Pi|))$ sample complexity bound for learning reactive policies in general POMDPs as $\log(|\Pi|) = H\log(K)$ in the construction, but the number of samples required is exponential in $H$. 
The lower bound instance provides essentially no instantaneous feedback and therefore forces the agent to reason over $K^H$ paths independently.

\textbf{Predictability of $\qs$:}
The assumption underlying the empirical successes in RL is that the $\qs$ function can be well-approximated by some large set of functions $\Fcal$.
To formalize this assumption, note that for some POMDPs, we may be able to write $\qs$ as a function of the observed history $(x_1,a_1,r_1(a_1), \ldots, x_h)$ at time $h$.
For example, this is always true in deterministic-transition POMDPs, since the sequence of previous actions encodes the state and $\qs$ as in Eq.~\eqref{eqn:qstar} depends only on the state, the current observation, and the proposed action.
In the \emph{realizable} setting, we have access to a collection of functions $\Fcal$ mapping the observed history to $[0,1]$, and we assume that $\qs \in \Fcal$. 

Unfortunately, even with realizability, learning in POMDPs can require exponentially many samples.
\begin{proposition}
\label{prop:realizable_lb}
Fix $H,K \in \NN$ with $K \ge 2$ and $\epsilon \in (0,\sqrt{1/8})$.
For any algorithm, there exists a LE-POMDP with time horizon $H$, $K$
actions, and $2H$ total states; a class of predictors $\Fcal$ with
$|\Fcal| = K^H$ and $\qs \in \Fcal$; and a constant $c \ge 0$ such
that the probability that the algorithm outputs a policy $\hat{\pi}$
with $V(\hat{\pi}) > V^\star - \epsilon$ after collecting $T$
trajectories is at most $2/3$ for all $T \le cK^H/\epsilon^2$.
\end{proposition}
As with Proposition~\ref{prop:agnostic_lb}, this lower bound precludes a $\textrm{Poly}(K,M,H,\log(|\Pi|))$ sample complexity bound for learning POMDPs with realizability.
The lower bound shows that even with realizability, the agent may have to reason over $K^H$ paths independently since the functions can depend on the entire history.
Proofs of both lower bounds here are deferred to Appendix~\ref{app:lower_bounds}.

Both lower bounds use POMDPs with deterministic transitions and an extremely small observation space. 
Consequently, even learning in deterministic-transition POMDPs requires further assumptions.

\subsection{Main Assumptions}
As we have seen, neither restricting to reactive policies, nor
imposing realizability enable tractable learning in POMDPs on their
own.  Combined however, we will see that sample-efficient learning is
possible, and the combination of these two assumptions is precisely
how we characterize our model.  Specifically, we study POMDPs for
which $\qs$ can be realized by a predictor that uses only the current
observation and proposed action.

\begin{assum}[\emph{Reactive Value Functions}]
\label{as:qs}
We assume that for all $x \in \Xspace, a \in \Aspace$ and any two
state $s,s'$ such that $\dist{s}{x},\dist{s'}{x} > 0$, we have
\mbox{$\qs_s(x,a) = \qs_{s'}(x,a)$}.
\end{assum}

The restriction on $\qs$ implies that the optimal policy is reactive and also that the optimal predictor of long-term reward depends only on the current observation.
In the following section, we describe how this condition relates to other RL models in the literature. 
We first present a natural example. 

\begin{example}[\emph{Disjoint observations}] The simplest example is one where each state $s$ can be identified with a subset
  $\Xspace_s$ with $\dist{s}{x} > 0$ only for $x \in \Xspace_s$ and
  where $\Xspace_s \cap \Xspace_{s'} = \emptyset$ when $s \ne s'$. A
  realized observation then uniquely identifies the underlying state
  $s$ so that Assumption~\ref{as:qs} trivially holds, but this mapping from
  $s$ to $\Xspace_s$ is unknown to the agent. Thus, the problem cannot
  be easily reduced to a small-state MDP.  This setting is quite
  natural in several robotics and navigation tasks, where the visual
  signals are rich enough to uniquely identify the agent's position
  (and hence state).  It also applies to video game playing, where the
  raw pixel intensities suffice to decode the game's memory state, but
  learning this mapping is challenging.
  \label{ex:disjoint}
\end{example}


Thinking of $x$ as the state, the above example is an MDP with
infinite state space but with structured transition operator.  While
our model is more general, we are primarily motivated by these
infinite-state MDPs, for which the reactivity assumptions are
completely non-restrictive. 
For infinite-state MDPs, our model describes a particular structure on
the transition operator that we show enables efficient learning.
We emphasize that our focus is not on partial observability issues.

As we are interested in understanding function approximation, we make
a realizability assumption.
\begin{assum}[\emph{Realizability}]
\label{as:realize}
We are given access to a class of predictors $\Fcal \subseteq (\Xcal
\times \Acal \rightarrow [0,1])$ of size $|\Fcal| = N$ and assume that
$Q^\star = f^\star \in \Fcal$.  We identify each predictor $f$ with a
policy $\pi_f(x) \triangleq \argmax_a f(x,a)$.  Observe that the
optimal policy is $\pi_{f^\star}$ which satisfies $V(\pi_{f^\star}) =
V^\star$.
\end{assum}

Assumptions~\ref{as:qs} and~\ref{as:realize} exclude the lower bounds
from Propositions~\ref{prop:agnostic_lb}
and~\ref{prop:realizable_lb}. Our algorithm requires one further
assumption. 

\begin{assum}[\emph{Deterministic Transitions}]
  \label{as:det}
We assume that the transition model is deterministic.  This means that
the starting distribution $\Trans_1$ is a point-mass on some state
$s_1$ and $\Trans: (\Sspace \times \Aspace) \rightarrow \Sspace$.
\end{assum}

Even with deterministic transitions, learning requires systematic
global exploration that is unaddressed in previous work.  Recall that
the lower bound constructions for Propositions~\ref{prop:agnostic_lb}
and~\ref{prop:realizable_lb} actually use deterministic transition
POMDPs.  Therefore, deterministic transitions combined with either the
reactive or the realizability assumption by itself still precludes
tractable learning.  Nevertheless, we hope to relax this final
assumption in future work.

More broadly, this model provides a framework to reason
about reinforcement learning with function approximation.  This is
highly desirable as such approaches are the empirical
state-of-the-art, but the limited supporting theory provides little
advice on systematic global exploration.

%% file: related.tex
\subsection{Connections to Other Models and Techniques}
\label{sec:related}

The above model is closely related to several well-studied models in the
literature, namely:

\textbf{Contextual Bandits:} If $H=1$, then our model reduces to
stochastic contextual bandits~\cite{langford2008epoch,
  dudik2011efficient}, a well-studied simplification of the general
reinforcement learning problem. The main difference is that the choice
of action \emph{does not} influence the future observations (there is
only one state), and algorithms do not need to perform long-term
planning to obtain low sample complexity.

\textbf{Markov Decision Processes:} If $\Xspace=\Sspace$ and
$\dst{s}(x)$ for each state $s$ is concentrated on $s$, then
our model reduces to small-state MDPs, which can be efficiently solved by tabular
approaches~\cite{kearns2002near,brafman2003r,strehl2006pac}.  The key
differences in our setting are that the observation space $\Xspace$ is
extremely large or infinite and the underlying state is unobserved, so
tabular methods are not viable and algorithms need to
\emph{generalize} across observations.

When the number of states is large, existing methods typically require
exponentially many samples such as the $\order(\Asize^H)$ result
of~\citet{kearns2002sparse}.  Others depend poorly on the complexity
of the policy set or scale linearly in the size of a covering over the
state space~\cite{kakade2003exploration,jong2007model,pazis2016efficient}.  Lastly,
policy gradient methods avoid dependence on size of the state space,
but do not achieve global
optimality~\cite{sutton1999policy,kakade2002approximately} in theory
and in practice, unlike our algorithm which is guaranteed to find the
globally optimal policy.





\textbf{POMDPs:} By definition our model is a POMDP where the $\qs$
function is consistent across states.  This restriction implies that
the agent does not have to reason over belief states as is required in
POMDPs.  There are some sample complexity guarantees for learning in
arbitrarily complex POMDPs, but the bounds we are aware of are quite
weak as they scale linearly with
$|\Pi|$~\cite{kearns1999approximate,mansour1999reinforcement}, or
require discrete observations from a small
set~\cite{azizzadenesheli2016reinforcement}.

\textbf{State Abstraction:} State abstraction (see~\cite{li2006towards} for a survey) focuses on understanding what
optimality properties are preserved in an MDP after the state space is
compressed.
While our model does have a small number of underlying states, they do not necessarily admit non-trivial state
abstractions that are easy to discover (i.e. that do not
amount to learning the optimal behavior) as the optimal behavior can
depend on the observation in an arbitrary manner. Furthermore, most
sample complexity results cannot search over large abstraction sets
(see e.g.~\citet{jiang2015abstraction}), limiting their scope.

\textbf{Function Approximation:} Our approach uses function
approximation to address the generalization problem implicit in our
model.  Function approximation is the empirical state-of-the-art in
reinforcement learning~\cite{mnih2015human}, but theoretical analysis
has been quite limited.  Several authors have studied linear or more
general function approximation
(See~\cite{tsitsiklis1997analysis,perkins2002convergent,baird1995residual}),
but none of these results give finite sample bounds, as they do not
address the exploration question.  Li and
Littman~\cite{li2010reducing} do give finite sample bounds, but they
assume access to a ``Knows-what-it-knows" (KWIK) oracle, which cannot
exist even for simple problems.  Other theoretical results either make
stronger realizability assumptions (c.f.,~\cite{antos2008learning}) or
scale poorly with problem parameters (e.g., polynomial in the number
of functions~\cite{nguyen2013competing} or the size of the observation
space~\cite{pazis2016efficient}).


%% file: approach.tex
\section{The Result}
\label{sec:approach}
We consider the task of Probably Approximately Correct (PAC) learning
the models defined in Section~\ref{sec:model}. 
Given $\Fcal$ (Assumption~\ref{as:realize}), we say that an algorithm
PAC learns our model if for any $\epsilon,\delta\in (0,1)$, the
algorithm outputs a policy $\hat{\pi}$ satisfying $V(\hat{\pi}) \ge
V^\star - \epsilon$ with probability at least $1-\delta$.
The \emph{sample complexity} is a
function $n: (0,1)^2 \rightarrow \NN$ such that for any
$\epsilon,\delta \in (0,1)$, the algorithm returns an
$\epsilon$-suboptimal policy with probability at least $1-\delta$
using at most $n(\epsilon,\delta)$ episodes. We refer to a
$\textrm{Poly}(\Ssize,\Asize,H,1/\epsilon,\log N,\log(1/\delta))$
sample complexity bound as polynomial in all relevant parameters.
Notably, there should be no dependence on $|\Xspace|$, which
may be infinite.

\subsection{The Algorithm}
Before turning to the algorithm, it is worth clarifying some
additional notation.  Since we are focused on the deterministic
transition setting, it is natural to think about the environment as an
exponentially large search tree with fan-out $\Asize$ and depth $H$.
Each node in the search tree is labeled with an (unobserved) state $s
\in \Sspace$, and each edge is labeled with an action $a \in \Aspace$,
consistent with the transition model.  A path $p \in \Aspace^\star$ is
a sequence of actions from the root of the search tree, and we also
use $p$ to denote the state reached after executing the path $p$ from
the root.  Thus, $D_p$ is the observation distribution of the state at
the end of the path $p$.
We use $p \circ a$ to denote a path formed by executing all actions in
$p$ and then executing action $a$, and we use $|p|$ to denote the
length of the path.  Let $\treeroot$ denote the empty path, which
corresponds to the root of the search tree.

The pseudocode for the algorithm, which we call \alglong (\alg), is
displayed in Algorithm~\ref{alg:dfs_det_big_obs} (See also
Appendix~\ref{app:pseudocode}).  \alg has two main components: a
depth-first-search routine with a learning step
(step~\ref{algline:elim} in Algorithm~\ref{alg:dfs_learn}) and an
on-demand exploration technique
(steps~\ref{algline:beg_demand}-\ref{algline:end_demand} in
Algorithm~\ref{alg:dfs_det_big_obs}). The high-level idea of the
algorithm is to eliminate regression functions that do not meet
Bellman-like consistency properties of the $Q^\star$ function. We now
describe both components and their properties in detail.

\begin{algorithm}[t]
\begin{algorithmic}[1]
\State $\Fcal \gets \dfslearn(\treeroot,\Fcal,\epsilon,\delta/2)$. 
\State Choose any $f \in \Fcal$. Let $\hat{V}^\star$ be a Monte Carlo estimate of $V^f(\treeroot,\pi_f)$. (See Eq.~\eqref{eqn:vfdef})
\State Set $\epsdemand = \epsilon/2, \ndemand = \frac{32\log(12MH/\delta)}{\epsilon^2}$ and $\ndemandtwo = \frac{8\log(6MH/\delta)}{\epsilon}$. 
\While{\true}
\State Fix a regressor $f \in \Fcal$. \label{algline:beg_demand}
\State Collect $\ndemand$ trajectories according to $\pi_f$ and estimate $V(\pi_f)$ via Monte-Carlo estimate $\hat{V}(\pi_f)$.\label{algline:mcest}
\State If $|\hat{V}(\pi_f) - \hat{V}^\star| \le \epsdemand$, return $\pi_f$.
\State Otherwise update $\Fcal$ by calling \dfslearn$(p,\Fcal,\epsilon,\frac{\delta}{6MH^2\ndemandtwo})$ on each of the $H-1$ prefixes $p$ of each of the first $\ndemandtwo$ paths collected in step \ref{algline:mcest}. \label{algline:end_demand}
\EndWhile
\end{algorithmic}
\caption{\alglong: \alg$(\Fcal,\epsilon,\delta)$}
\label{alg:dfs_det_big_obs}
\end{algorithm}

\begin{algorithm}[t]
\begin{algorithmic}[1]
\State Set $\esterr = \frac{\epsilon}{320H^2\sqrt{K}}$ and $\epstest = 20(H-|p|-5/4)\sqrt{K}\esterr$.
\For{$a \in \Acal$, if not $\statelearned(p\circ
  a,\Fspace,\epstest,\esterr,\frac{\delta/2}{\Ssize \Asize H})$} \label{algline:check_con}
\State $\Fspace \gets \dfslearn(p \circ a, \Fspace, \epsilon,\delta)$. 
\EndFor
\State Collect $\ntrain = \frac{24}{\esterr^2}\log\left(\frac{8MH\Fsize}{\delta}\right)$ observations $(x_i,a_i,r_i)$ where $(x_i,r_i') \sim \dst{p}$, $a_i$ is chosen uniformly at random, and $r_i = r_i'(a_i)$.
\State Return $\left\{f \in \Fspace: \tilde{R}(f) \le \min_{f'
  \in \Fspace}\tilde{R}(f') + 2\esterr^2 +
\frac{22\log(4MH\Fsize/\delta)}{\ntrain}\right\}$, $\tilde{R}(f)$ defined in
Eq.~\eqref{eq:td_sq_loss}. \label{algline:elim} 
\end{algorithmic}
\caption{\dfslearn$(p,\Fcal,\epsilon,\delta)$}
\label{alg:dfs_learn}
\end{algorithm}

\begin{algorithm}[t]
\begin{algorithmic}[1]
\State Set $\ntest = \frac{2}{\esterr^2}\log(2\Fsize/\delta)$. Collect $\ntest$ observations $x_i \sim \dst{p}$.
\State Compute for each function, $\fpvhat{f}{p} = \frac{1}{\ntest}\sum_{i=1}^{\ntest} f(x_i,\pi_{f}(x_i))$.\label{algline:vhat}
\State Return $\mathbf{1}\left[|\fpvhat{f}{p} - \fpvhat{g}{p}| \le \epstest\ \forall f,g \in \Fspace\right]$.
\end{algorithmic}
\caption{$\statelearned(p, \Fspace, \epstest, \esterr, \delta)$}
\label{alg:state_learned}
\end{algorithm}

\textbf{The DFS routine}: When the DFS routine, displayed in
Algorithm~\ref{alg:dfs_learn}, is run at some path $p$, we first
decide whether to recursively expand the descendants $p\circ a$ by
performing a \emph{consensus test}.  Given a path $p'$, this test,
displayed in Algorithm~\ref{alg:state_learned}, computes estimates of
\emph{value predictions},
\begin{equation}
  V^f(p',\pi_f) \triangleq \EE_{x \sim D_{p'}}f(x,\pi_f(x)),
  \label{eqn:vfdef}
\end{equation} 
for all the surviving regressors. These value predictions are easily
estimated by collecting many observations after rolling in to $p'$ and
using empirical averages (See line~\ref{algline:vhat} in
Algorithm~\ref{alg:state_learned}).  If all the functions agree on
this value for $p'$ the DFS need not visit this path.

After the recursive calls, the DFS routine performs the
\emph{elimination step} (line~\ref{algline:elim}).  When this step is
invoked at path $p$, the algorithm collects $\ntrain$ observations
$(x_i,a_i,r_i)$ where $(x_i,r_i') \sim \dst{p}$, $a_i$ is chosen
uniformly at random, and $r_i = r_i'(a_i)$ and eliminates regressors
that have high empirical risk,
\begin{align}
\tilde{R}(f) \triangleq \frac{1}{\ntrain}\sum_{i=1}^{\ntrain} (f(x_i,a_i) - r_i - \hat{V}^f(p\circ a_i, \pi_f))^2 \label{eq:td_sq_loss}.
\end{align}

\textbf{Intuition for DFS}: This regression problem is motivated by
the realizability assumption and the definition of $\qs$ in
Eq.~\eqref{eqn:qstar}, which imply that at path $p$ and for all
actions $a$,
\begin{align}
f^\star(x,a) &= \EE_{r \sim \dst{p|x}}r(a) + V(p\circ a,\pi_{f^\star})
= \EE_{r \sim \dst{p|x}} r(a) + \EE_{x'\sim \dst{p\circ a}}
f^\star(x',\pi_{f^\star}(x')). \label{eq:f_consistency}
\end{align}
Thus $f^\star$ is consistent between its estimate at the current state
$s$ and the future state $s' = \Trans(s,a)$.

The regression problem~\eqref{eq:td_sq_loss} is essentially a finite
sample version of this identity.  However, some care must be taken as
the target for the regression function $f$ includes $V^f(p \circ
a,\pi_f)$, which is $f$'s value prediction for the future.  The fact
that the target differs across functions can cause instability in the
regression problem, as some targets may have substantially lower
variance than $f^\star$'s.  To ensure correct behavior, we must obtain
high-quality future value prediction estimates, and so, we re-use the
Monte-Carlo estimates $\hat{V}^f(p\circ a, \pi_f)$ in
Eq.~\eqref{eqn:vfdef} from the consensus tests. Each time we perform
elimination, the regression targets are close for all considered $f$
in Equation~\eqref{eq:td_sq_loss} owing to consensus being satisfied
at the successor nodes in Step~\ref{algline:check_con} of
Algorithm~\ref{alg:dfs_learn}.

Given consensus at all the descendants, each elimination step
inductively propagates learning towards the start state by ensuring
the following desirable properties hold: (i) $f^\star$ is not
eliminated, (ii) consensus is reached at $p$, and (iii) surviving
policies choose good actions at $p$.  Property (ii) controls the
sample complexity, since consensus tests at state $s$ return true once
elimination has been invoked on $s$, so DFS avoids exploring the
entire search space. Property (iii) leads to the PAC-bound; if we have
run the elimination step on all states visited by a policy, that
policy must be near-optimal.

To bound the sample complexity of the DFS routine, since there are
$\Ssize$ states per level and the consensus test returns true once
elimination has been performed, we know that the DFS does not visit a
large fraction of the search tree.  Specifically, this means DFS is
invoked on at most $\Ssize H$ nodes in total, so we run elimination at
most $\Ssize H$ times, and we perform at most $\Ssize \Asize H$
consensus tests.  Each of these operations requires polynomially many
samples.

The elimination step is inspired by the RegressorElimination algorithm
of Agarwal et. al~\cite{agarwal2012contextual} for contextual bandit
learning in the realizable setting.  In addition to forming a
different regression problem, RegressorElimination carefully chooses
actions to balance exploration and exploitation which leads to an
optimal regret bound.  In contrast, we are pursuing a PAC-guarantee
here, for which it suffices to focus exclusively on exploration.

\textbf{On-demand Exploration}: While DFS is guaranteed to estimate
the optimal value $V^\star$, it unfortunately does not identify the
optimal policy.  For example, if consensus is satisfied at a state $s$
without invoking the elimination step, then each function accurately
predicts the value $V^\star(s)$, but the associated policies are not
guaranteed to achieve this value.  To overcome this issue, we use an
\emph{on-demand exploration} technique in the second phase of the
algorithm (Algorithm~\ref{alg:dfs_det_big_obs},
steps~\ref{algline:beg_demand}-\ref{algline:end_demand}).

At each iteration of this phase, we select a policy $\pi_f$ and
estimate its value via Monte Carlo sampling.  If the policy has
sub-optimal value, we invoke the DFS procedure on many of the paths
visited.  If the policy has near-optimal value, we have found a good
policy, so we are done.  This procedure requires an accurate estimate
of the optimal value, which we already obtained by invoking the DFS
routine at the root, since it guarantees that all surviving regressors
agree with $f^\star$'s value on the starting state distribution.
$f^\star$'s value is precisely the optimal value.

\textbf{Intuition for On-demand Exploration}: Running the elimination
step at some path $p$ ensures that all surviving regressors take good
actions at $p$, in the sense that taking one action according to any
surviving policy and then behaving optimally thereafter achieves
near-optimal reward for path $p$.  This does not ensure that all
surviving policies achieve near-optimal reward, because they may take
highly sub-optimal actions after the first one.  On the other hand, if
a surviving policy $\pi_f$ visits only states for which the
elimination step has been invoked, then it must have near-optimal
reward.  More precisely, letting $L$ denote the set of states for
which the elimination step has been invoked (the ``learned" states),
we prove that any surviving $\pi_f$ satisfies
\begin{align*}
V^\star - V(\pi_f) \le \epsilon/8 + \PP\left[ \pi_f \textrm{ visits a
    state} s \notin L \right]
\end{align*}

Thus, if $\pi_f$ is highly sub-optimal, it must visit some unlearned
states with substantial probability. By calling \dfslearn on the paths
visited by $\pi_f$, we ensure that the elimination step is run on at
least one unlearned states.  Since there are only $\Ssize H$ distinct
states and each non-terminal iteration ensures training on an
unlearned state, the algorithm must terminate and output a
near-optimal policy.

Computationally, the running time of the algorithm may be $O(N)$,
since eliminating regression functions according to
Eq.~\eqref{eq:td_sq_loss} may require enumerating over the class and
the consensus function requires computing the maximum and minimum of
$N$ numbers, one for each function. This may be intractably slow for
rich function classes, but our focus is on statistical efficiency, so
we ignore computational issues here.

\subsection{The PAC Guarantee}
Our main result certifies that \alg PAC-learns our models with
polynomial sample complexity.
\begin{theorem}[PAC bound]
\label{thm:det_dfs_big_obs}
For any $(\epsilon,\delta) \in (0,1)$ and under
Assumptions~\ref{as:qs},~\ref{as:realize}, and~\ref{as:det}, with
probability at least $1-\delta$, the policy $\pi$ returned by \alg is
at most $\epsilon$-suboptimal.  Moreover, the number of episodes
required is at most
\begin{align*}
\otil\left(\frac{\Ssize
  H^6\Asize^2}{\epsilon^3}\log(\Fsize/\delta)\log(1/\delta)\right).
\end{align*}
\end{theorem}

This result uses the $\otil$ notation to suppress logarithmic
dependence in all parameters except for $\Fsize$ and $\delta$.  The
precise dependence on all parameters can be recovered by examination
of our proof and is shortened here simply for clarity.
See Appendix~\ref{app:full_proof} for the full proof of the result.

This theorem states that \alg produces a policy that is at most
$\epsilon$-suboptimal using a number of episodes that is polynomial in
all relevant parameters.  To our knowledge, this is the first
polynomial sample complexity bound for reinforcement learning with
infinite observation spaces, without prohibitively strong assumptions
(e.g.,~\cite{antos2008learning,nguyen2013competing,pazis2016efficient}).
We also believe this is the first finite-sample guarantee for
reinforcement learning with general function approximation without
prohibitively strong assumptions (e.g.,~\cite{antos2008learning}).

Since our model generalizes both contextual bandits and MDPs, it is worth comparing the sample complexity bounds. 
\begin{enums}
\item In contextual bandits, we have $M=H=1$ so that the sample
  complexity of \alg is
  $\otil(\frac{K^2}{\epsilon^3}\log(\Fsize/\delta)\log(1/\delta))$, in
  contrast with known $\otil(\frac{K}{\epsilon^2}\log(\Fsize/\delta))$
  results.
\item Prior results establish the sample complexity for learning
  layered episodic MDPs with deterministic transitions is
  $\otil(\frac{\Ssize \Asize
    \textrm{poly}(H)}{\epsilon^2}\log(1/\delta))$~\cite{dann2015sample,reveliotis2007efficient}.
\end{enums}

Both comparisons show our sample complexity bound may be suboptimal in
its dependence on $K$ and $\epsilon$.  Looking into our proof, the
additional factor of $K$ comes from collecting observations to
estimate the value of future states, while the additional $1/\epsilon$
factor arises from trying to identify a previously unexplored state.
In contextual bandits, these issues do not arise since there is only
one state, while, in tabular MDPs, they can be trivially resolved as the
states are observed.  Thus, with minor modifications, \alg can avoid
these dependencies for both special cases.  In addition, our bound
disagrees with the MDP results in the dependence on the policy
complexity $\log(N)$; which we believe is unavoidable when working
with rich observation spaces.

Finally, our bound depends on the number of states $M$ in the worst
case, but the algorithm actually uses a more refined notion. Since the states
are unobserved, the algorithm considers two states distinct only if
they have reasonably different value functions, meaning learning on
one does not lead to consensus on the other. Thus, a more
distribution-dependent analysis defining states through the function
class is a promising avenue for future work.

%% file: discussion.tex
\section{Discussion}
\label{sec:discussion}

This paper introduces a new model in which it is possible to design
and analyze principled reinforcement learning algorithms engaging in
global exploration.  As a first step, we develop a new algorithm and
show that it learns near-optimal behavior under a
deterministic-transition assumption with polynomial sample complexity.
This represents a significant advance in our understanding of
reinforcement learning with rich observations.
However, there
are major open questions:
\begin{enums}
\item Do polynomial sample bounds for this model with stochastic
  transitions exist?
\item Can we design an algorithm for learning this model that is both
  computationally and statistically efficient?  The sample complexity
  of our algorithm is logarithmic in the size of the function class
  $\Fcal$ but uses an intractably slow enumeration of these functions.
\end{enums}
Good answers to both of these questions may yield new practical reinforcement learning algorithms.


%% file: appendix.tex
\section{The Lower Bounds}
\label{app:lower_bounds}

\begin{theorem}[Lower bound for best arm identification in stochastic bandits]
\label{thm:bandit_lb}
For any $K \ge 2$ and $\epsilon \le \sqrt{1/8}$ and any best-arm identification algorithm, there exists a
multi-armed bandit problem for which the best arm $i^\star$ is
$\epsilon$ better than all others, but for which the estimate $\hat{i}$ of the best arm must have $\PP[\hat{i}
  \ne i^\star] \ge 1/3$ unless the number of samples collected $T$ is
at least $\frac{K}{72\epsilon^2}$.
\end{theorem}
\begin{proof}
  The proof is essentially the same as the regret lower bound for
  stochastic multi-armed bandits
  from~\citet{auer2002nonstochastic}. Since we want the lower bound
  for best arm identification instead of regret, we include a full
  proof for completeness.
  
  Following~\citet{auer2002nonstochastic}, the lower bound
  instance is drawn uniformly from a family of multi-armed bandit
  problems with $K$ arms each. There are $K$ problems in the family,
  and each one is parametrized by the optimal arm $i^\star$.  For the
  $i^{\star\textrm{th}}$ problem, arm $i^\star$ produces rewards drawn from
  $\textrm{Ber}(1/2+\epsilon)$ while all other arms produce rewards from
  $\textrm{Ber}(1/2)$.  Let $\PP_{i^\star}$ denote the reward
  distribution for the $i^{\star\textrm{th}}$ bandit problem, so
  that $\PP_{i^\star}(\cdot | a = i^\star) = \textrm{Ber}(1/2+\epsilon)$
  and $\PP_{i^\star}(\cdot | a \ne i^\star) = \textrm{Ber}(1/2)$.  Let
  $\PP_0$ denote the reward distribution where all arms receive
  $\textrm{Ber}(1/2)$ rewards.

Since the environment is stochastic, any randomized algorithm is just
a distribution over deterministic ones, and it therefore suffices to
consider only deterministic algorithms.  More precisely, a randomized
algorithm uses some random bits $z$ and for each choice, the algorithm
itself is deterministic.  If we lower bound $\PP_{i^\star}[\hat{i} \ne
  i^\star | z] $ for all $z$, then we also obtain a lower bound after
taking expectation.

A deterministic algorithm can be specified as a sequence of mappings
$\psi_t: \{0,1\}^t \rightarrow [K]$ with the interpretation of
$\psi_T$ as the estimate of the best arm.  Note that $\psi_0$ is the
first arm chosen, which does not depend on any of the observations.
The algorithm can be specified this way since the sequence of actions
played can be inferred by the sequence of observed rewards.  Let
$\PP_{i^\star,\psi}$ denote the distribution over all $T$ rewards when
$i^\star$ is the optimal arm and actions are selected according to
$\psi$.  We are interested in bounding the error event
$\PP_{i^\star,\psi}[\psi_T \ne i^\star]$.

We first prove,
\begin{align*}
\PP_{i^\star,\psi}[\psi_T = i^\star] - \PP_{0,\psi}[\psi_T = i^\star]
\le \frac{1}{2}\sqrt{\EE_{0,\psi}[N_{i^\star}]\log\frac{1}{1-4\epsilon^2}},
\end{align*}
where $N_i$ is the number of times $\psi$ plays action $i$ over the
course of $T$ rounds.  $N_i$ is a random variable since it depends on
the sequence of observations, and here we take expectation with
respect to $\PP_0$.

To prove this statement, notice that,
\begin{align*}
\left|\PP_{i^\star,\psi}[\psi_T = i^\star] - \PP_{0,\psi}[\psi_T =
  i^\star]\right| \le \|P_{i^\star,\psi} - P_{0,\psi}\|_{\textrm{TV}}
\le \sqrt{\frac{1}{2}KL(P_{0,\psi} || P_{i^\star,\psi})}~.
\end{align*}
The first inequality is by definition of the total variation distance,
while the second is Pinsker's inequality.  We are left to bound the KL
divergence.  To do so, we introduce notation for sequences.  For any
$t \in \NN$, we use $r_{1:t} \in \{0,1\}^t$ to denote the binary
reward sequence of length $t$.  The KL divergence is
\begin{align*}
KL(P_{0,\psi} || P_{i^\star,\psi}) &= \sum_{r_{1:T} \in \{0,1\}^T}
P_{0,\psi}(r_{1:T})
\log\left(\frac{P_{0,\psi}(r_{1:T})}{P_{i^\star,\psi}(r_{1:T})}\right)\\ &
= \sum_{t=1}^T\sum_{r_{1:t} \in \{0,1\}^t}P_{0,\psi}(r_{1:t})
\log\left(\frac{P_{0,\psi}(r_t|r_{1:t-1})}{P_{i^\star,\psi}(r_t|r_{1:t-1})}\right)\\ &
= \sum_{t=1}^T\sum_{r_{1:t-1}: a_t = i^\star}P_{0,\psi}(r_{1:t-1})
\left(\sum_{x\in\{0,1\}} P_{0,\psi}(x)\log\left(\frac{P_{0,\psi}(x|a_t
  = i^\star)}{P_{i^\star,\psi}(x|a_t = i^\star)}\right)\right)~,
\end{align*}
where $a_t$ is the chosen action at time $t$. To arrive at the second
line we use the chain rule for KL-divergence.  The third line is based
on the fact that if $a_t \ne i^\star$, then the log ratio is zero,
since the two conditional distributions are identical.  Continuing
with straightforward calculations, we have
\begin{align*}
KL(P_{0,\psi} || P_{i^\star,\psi}) &= \sum_{t=1}^T\sum_{r_{1:t-1}: a_t
  = i^\star}P_{0,\psi}(r_{1:t-1})
\left(\frac{1}{2}\log\left(\frac{1/2}{1/2-\epsilon}\right) +
\frac{1}{2}\log\left(\frac{1/2}{1/2+\epsilon}\right)\right)\\ & =
\left(-\frac{1}{2}\log(1-4\epsilon^2)\right)\sum_{t=1}^T\sum_{r_{1:t-1}:a_t=i^\star}
P_{0,\psi}(r_{1:t-1})\\ & =
\left(-\frac{1}{2}\log(1-4\epsilon^2)\right)\sum_{t=1}^T\PP_{0,\psi}[a_t
  = i^\star].
\end{align*}
This proves the sub-claim, which follows the same argument as as Auer et. al~\cite{auer2002nonstochastic}.

To prove the final result, we take expectation over the problem $i^\star$.
\begin{align*}
\frac{1}{K}\sum_{i^\star=1}^K\PP_{i^\star,\psi}[\psi_T = i^\star] &\le \frac{1}{K}\sum_{i^\star=1}^K\PP_{0,\psi}[\psi_T = i^\star] + \frac{1}{2K}\sum_{i^\star=1}^K\sqrt{\EE_{0,\psi}[N_{i^\star}]\log\frac{1}{1-4\epsilon^2}}\\
& \le \frac{1}{K} + \frac{1}{2}\sqrt{\frac{-\log(1-4\epsilon^2)}{K}\EE_{0,\psi}\sum_{i^\star=1}^KN_{i^\star}}
 \le \frac{1}{K} + \frac{1}{2}\sqrt{\frac{-\log(1-4\epsilon^2) T}{K}}.
\end{align*}
If $4\epsilon^2 \le 1/2$ then $-\log(1-4\epsilon^2) \le 8\epsilon^2$.
This follows by the Taylor expansion of $-\log(1-x)$,
\begin{align*}
-\log(1-x) = \sum_{i=1}^{\infty} \frac{x^i}{i} \le x\left( \sum_{i=0}^{\infty}\frac{2^{-i}}{i+1}\right) \le x \sum_{i=0}^\infty 2^{-i} = 2x.
\end{align*}
The inequality here uses the assumption that $x \le 1/2$.

Thus, whenever $\epsilon \le \sqrt{1/8}$ and $T  \le \frac{K}{72\epsilon^2}$, this number is smaller than $2/3$, since we restrict to the cases where $K\ge 2$.
This is the success probability, so the failure probability is at least $1/3$, which proves the result.
\end{proof}

\subsection{The construction}
Here we design a family of POMDPs for both lower bounds. As with
multi-armed bandits above, the lower bound will be realized by
sampling a POMDP from a uniform distribution over this family of
problems. Fix $H$ and $K$ and pick a single $x_h \in \Xcal$ for each
level $h \in [H]$ so that $x_h \ne x_{h'}$ for all pairs $h \ne h'$.
For each level there are two states $g_h$ and $b_h$ for ``good'' and
``bad.''  The observation marginal distribution $D_{g_h} = D_{b_h}$ is
concentrated on $x_h$ for each level $h$, so the observations provide
no information about the underlying state.  Rewards for all levels
except for $h=H$ are zero.

Each POMDPs in the family corresponds to a path $p^\star = (a_1^\star,
\ldots, a_H^\star) \in K^H$.  The transition function for the POMDP
corresponding to the path $p^\star$ is,
\begin{align*}
\Gamma(g_h,a_h^\star) &\triangleq g_{h+1}\\
\Gamma(g_h,a) &\triangleq b_{h+1} \mbox{ if } a \ne a_h^\star\\
\Gamma(b_h,a) &\triangleq b_{h+1}\ \forall\ a.
\end{align*}
The reward is drawn from $\textrm{Ber}(1/2+\epsilon)$ if the last state is $g_H$ and if the last action is $a_H^\star$.
For all other outcomes the reward is drawn from $\textrm{Ber}(1/2)$. 
Observe that these models have deterministic transitions.

Clearly all of the models in this family are distinct, and there are
$K^H$ such models.  Moreover, since the observations $x_h$ provide no
information and only the final reward is non-zero, no
information is received until the full sequence of actions is
selected.  More formally, for any two policies $\pi, \pi'$, the KL
divergence between the distributions of observations and rewards
produced by the two policies is exactly the KL divergence between the
final rewards produced by the two policies. Therefore, the problem is
equivalent to a multi-armed bandit problem with $K^H$ arms, where the
optimal arm gets a $\textrm{Ber}(1/2+\epsilon)$ reward while all other
arms get a $\textrm{Ber}(1/2)$ reward. Thus, identifying a policy that
is no-more than $\epsilon$ suboptimal in this POMDP is
information-theoretically equivalent to identifying the best arm in
the stochastic bandit problem in Theorem~\ref{thm:bandit_lb} with
$K^H$ arms.  Applying that lower bound gives a sample complexity bound
of $\Omega(K^H/\epsilon^2)$.

\subsection{Proving both lower bounds}
To verify both lower bounds in Propositions~\ref{prop:agnostic_lb}
and~\ref{prop:realizable_lb}, we construct the policy and regressor
sets.  For Proposition~\ref{prop:agnostic_lb}, we need a set of
reactive policies such that finding the optimal policy has a large
sample complexity. To this end, we use the set of all $K^H$ mappings
from the $H$ observations to actions.  Specifically, each policy $\pi$
is identified with a sequence of $H$ actions $(a_1, \ldots, a_H)$ and
has $\pi(x_h) = a_h$.  These policies are reactive by definition since
they do not depend on any previous history, or state of the
world. Clearly there are $K^H$ such policies, and each policy is
optimal for exactly one POMDP defined above, namely $\pi_p$ is
optimal for the POMDP corresponding to the path $p$. Furthermore, in
the POMDP defined by $p$, we have $V(\pi_p) = 1/2+\epsilon$, whereas
$V(\pi) = 1/2$ for every other policy.  Consequently, finding the best
policy in the class is equivalent to identifying the best arm in this
family of problems. Taking a uniform mixture of problems in the family
as before, we reason that this requires at least
$\Omega(K^H/\epsilon^2)$ trajectories.

For Proposition~\ref{prop:realizable_lb}, we use a similar
construction. For each path $p = (a_1, \ldots, a_H)$, we associate a
regressor $f_p$ with,
\begin{align*}
f_p(\rho) \triangleq \frac{1}{2} + \epsilon\mathbf{1}[\rho \textrm{ is a prefix
    of } p].
\end{align*}
Here we use $\rho$ to denote the history of the interaction, which can
be condensed to a sequence of actions since the observations provide
no information.

Clearly for the POMDP parameterized by $p$, $f_p$ correctly maps the
history to future reward, meaning that the POMDP is realizable for
this regressor class. Relatedly, $\pi_{f_p}$ is the optimal policy for
the POMDP with optimal sequence $p$. Moreover, there are precisely
$K^H$ regressors.  As before, the learning objective requires
identifying the optimal policy and hence the optimal path, which
requires $\Omega(K^H/\epsilon^2)$ trajectories.

\section{Full Algorithm Pseudocode}
\label{app:pseudocode}
It is more natural to break the algorithm into more components for the
analysis.  This lets us focus on each component in isolation.

We first clarify some notation involving value functions.
For predictor $f$ and policy $\pi$, we use,
\begin{align*}
V^f(s,\pi) &\triangleq \EE_{x \sim D_s}[f(x,\pi(x))]\\
V(s,\pi) & \triangleq \EE_{x \sim D_s}[r(\pi(x)) + \EE_{s' \sim \Gamma(s,\pi(x))} V(s',\pi)].
\end{align*}
Recall that $V(s_{H+1},\pi) = 0$ for all $s_{H+1}$, which is a terminating state. 

We often use a path $p$ as the first argument, with the convention
that the associated state is the last one on the path.  This is
enabled by deterministic transitions.  If a state is omitted from
these functions, then it is assumed to be the start state or the root
of the search tree.  We also use $V^\star$ for the optimal value,
where by assumption we have $V^\star = V(\pi_{f^\star}) =
V^{f^\star}(\pi_{f^\star})$.  Finally, throughout the algorithm and
analysis, we use Monte Carlo estimates of these quantities, which we
denote as $\hat{V}^f, \hat{V}$, etc.

Pseudocode for the compartmentalized version of the algorithm is
displayed in Algorithm~\ref{alg:dfs_det_big_obs2} with subroutines
displayed as
Algorithms~\ref{alg:dfs_learn2},~\ref{alg:state_learned2},~\ref{alg:td_regressor_elimination},
and~\ref{alg:on_demand}.  The algorithm should be invoked as
$\alg(\Fspace,\epsilon,\delta)$ where $\Fspace$ is the given class
of regression functions, $\epsilon$ is the target accuracy and
$\delta$ is the target failure probability.  The two main components
of the algorithm are the \dfslearn and \ondemand routines.  \dfslearn
ensures proper invocation of the training step, \regelim, by verifying
a number of preconditions, while \ondemand finds regions of the search
tree for which training must be performed.

It is easily verified that this is an identical description of the algorithm. 

\begin{algorithm}[t]
\begin{algorithmic}[1]
\State $\Fcal \gets \dfslearn(\treeroot,\Fcal,\epsilon,\delta/2)$.
\State Choose any $f \in \Fcal$. Let $\hat{V}^\star$ be a Monte Carlo estimate of $V^f(\treeroot,\pi_f)$.
\State $f \gets \ondemand(\Fcal, \hat{V}^\star, \epsilon,\delta/2)$. 
\State Return $\pi_f$. 
\end{algorithmic}
\caption{\alglong: \alg$(\Fcal, \epsilon,\delta)$}
\label{alg:dfs_det_big_obs2}
\end{algorithm}

\begin{algorithm}[t]
\begin{algorithmic}[1]
\State Set $\esterr = \frac{\epsilon}{320H^2\sqrt{K}}$ and $\epstest = 20(H-|p|-5/4)\sqrt{K}\esterr$.
\For{$a \in \Acal$}
\If{Not $\statelearned(p\circ a,\Fspace,\epstest,\esterr,\frac{\delta/2}{\Ssize \Asize H})$}
\State $\Fspace \gets \dfslearn(p \circ a, \Fspace, \epsilon,\delta)$. \hfill  \# Recurse
\EndIf
\EndFor
\State $\hat{\Fspace} \gets \regelim\left(p,\Fspace, \esterr,\frac{\delta/2}{\Ssize H}\right)$. \hfill \# Learn in state $p$.
\State Return $\hat{\Fspace}$.
\end{algorithmic}
\caption{\dfslearn$(p,\Fcal,\epsilon,\delta)$}
\label{alg:dfs_learn2}
\end{algorithm}

\begin{algorithm}[t]
\begin{algorithmic}
\State Set $\ntest = 2\log(2\Fsize/\delta)/\esterr^2$.
\State Collect $\ntest$ observations $x_i \sim \dst{p}$.
\State Compute Monte-Carlo estimates for each value function,
\begin{align*}
\fpvhat{f}{p} = \frac{1}{\ntest}\sum_{i=1}^{\ntest} f(x_i,\pi_{f}(x_i)) \qquad \forall f \in \Fspace.
\end{align*}
\If {$|\fpvhat{f}{p} - \fpvhat{g}{p}| \le \epstest$ for all $f,g \in \Fspace$}
\State return \true.
\EndIf
\State Return \false.
\end{algorithmic}
\caption{$\statelearned(p, \Fspace, \epstest, \esterr, \delta)$}
\label{alg:state_learned2}
\end{algorithm}

\begin{algorithm}[t]
\begin{algorithmic}
\State Require estimates $\fpvhat{f}{p\circ a},\forall f \in \Fspace, a \in \Aspace$.
\State Set $\ntrain = 24\log(4\Fsize/\delta)/\esterr^2$.
\State Collect $\ntrain$ observations $(x_i,a_i,r_i)$ where $x_i \sim \dst{p}$, $a_i$ is chosen uniformly at random, and $r_i = r_i(a_i)$.
\State Update $\Fspace$ to
\begin{align}
&\left\{f \in \Fspace: \tilde{R}(f) \le \min_{f'
    \in \Fspace}\tilde{R}(f') + 2\esterr^2 +
  \frac{22\log(2\Fsize/\delta)}{\ntrain}\right\}, \nonumber \\
 &\mbox{with}~\tilde{R}(f) \triangleq \frac{1}{\ntrain}
  \sum_{i=1}^{\ntrain} (f(x_i,a_i) - r_i - \fpvhat{f}{p\circ a_i})^2. 
\label{eqn:elim}
\end{align}
\State Return $\Fspace$.
\end{algorithmic}
\caption{$\regelim(p, \Fspace, \esterr, \delta)$}
\label{alg:td_regressor_elimination}
\end{algorithm}

\begin{algorithm}[t]
\begin{algorithmic}
\State Set $\epsdemand = \epsilon/2, \ndemand = \frac{32\log(6MH/\delta)}{\epsilon^2}$ and $\ndemandtwo = \frac{8\log(3MH/\delta)}{\epsilon}$. 
\While{\true}
\State Fix a regressor $f \in \Fcal$.
\State Collect $\ndemand$ trajectories according to $\pi_f$ and estimate $V(\pi_f)$ via a Monte-Carlo estimate $\hat{V}(\pi_f)$.
\State If $|\hat{V}(\pi_f) - \hat{V}^\star| \le \epsdemand$, return $\pi_f$.
\State Otherwise update $\Fcal$ by calling \dfslearn$(p,\Fcal,\epsilon,\delta/(3MH^2\ndemandtwo))$ on each of the $H-1$ prefixes $p$ of each of the first $\ndemandtwo$ paths collected for the Monte-Carlo estimate.
\EndWhile
\end{algorithmic}
\caption{\ondemand$(\Fcal, \hat{V}^\star, \epsilon, \delta)$}
\label{alg:on_demand}
\end{algorithm}

\section{The Full Analysis}
\label{app:full_proof}
The proof of the theorem hinges on analysis of the the subroutines. 
We turn first to the \regelim~routine, for which we show the following guarantee.
Recall the definition,
\begin{align*}
\fpvfunc{f}{p} \triangleq \EE_{x \sim \dst{p}} f(x,\pi_f(x)).
\end{align*}

\begin{theorem}[Guarantee for \regelim]
Consider running \regelim~at path $p$ with regressors $\Fspace$, parameters $\esterr,\delta$ and with $\ntrain = 24 \log(4\Fsize/\delta)/\esterr^2$.
Suppose that the following are true:
\begin{enumerate}
\item \textbf{Estimation Precondition:} We have access to estimates $\hat{V}^f(p\circ a,\pi_f)$ for all $f \in \Fspace, a \in \Aspace$ such that, $|\fpvhat{f}{p\circ a}- \fpvfunc{f}{p\circ a}| \le \esterr$.
\item \textbf{Bias Precondition:} For all $f,g \in \Fspace$ and for all $a \in \Aspace$, $|\fpvfunc{f}{p\circ a} - \fpvfunc{g}{p\circ a}| \le \tdbias$. 
\end{enumerate}
Then the following hold simultaneously with probability at least $1-\delta$:
\begin{packed_enum}
\item $f^\star$ is retained by the algorithm.
\item \textbf{Bias Bound}:
\begin{align}
|\fpvfunc{f}{p} - \fpvfunc{g}{p}| \le 8\esterr\sqrt{K} + 2\esterr +
\tdbias.
\label{eqn:bias-bound}
\end{align}
\item \textbf{Instantaneous Risk Bound:} 
\begin{align}
\vs(p) - V^{f^\star}(p, \pi_f) \le 4\esterr\sqrt{2K} + 2\esterr + 2\tdbias.
\end{align}
\item \textbf{Estimation Bound}: Regardless of whether the
  preconditions hold, we have estimates $\fpvhat{f}{p}$ with,
\begin{align}
|\fpvhat{f}{p} - \fpvfunc{f}{p}| \le \frac{\esterr}{\sqrt{12}}.
\end{align}
\end{packed_enum}
The last three bounds hold for all surviving $f,g \in \Fspace$. 
\label{thm:td_regelim}
\end{theorem}

The theorem shows that, as long as we call \regelim with the two preconditions, then $f^\star$, the optimal regressor, always survives.
It also establishes a number of other properties about the surviving functions, namely that they agree on the value of this path (the bias bound) and that the associated policies take good actions from this path (the instantaneous risk bound). 
Note that the instantaneous risk bound is \emph{not} a cumulative risk bound.
The second term on the left hand side is the reward achieved by behaving like $\pi_f$ for one action but then behaving optimally afterwards.
The proof is deferred to Appendix~\ref{app:regelim}.

Analysis of the \statelearned~subroutine requires only standard concentration-of-measure arguments. 
\begin{theorem}[Guarantee for \statelearned]
\label{thm:state_learned}
Consider running \statelearned~on path $p$ with $\ntest =
2\log(2\Fsize/\delta)/\esterr^2$ and $\epstest \ge 2\esterr+\slbias$,
for some $\slbias > 0$.
\begin{packed_enum}
\item[(i)] With probability at least $1-\delta$, we have estimates
  $\fpvhat{f}{p}$ with \mbox{$|\fpvhat{f}{p} - \fpvfunc{f}{p}| \le
  \esterr$} \quad $\forall f \in \Fspace$.
\item[(ii)] If $|\fpvfunc{f}{p} - \fpvfunc{g}{p}| \le \slbias, \forall
  f,g \in \Fspace$, under the event (i), the algorithm returns \true.
\item[(iii)] If the algorithm returns \true, then under the event in
  (1), we have $|\fpvfunc{f}{p} - \fpvfunc{g}{p}| \le 2\esterr +
  \epstest$ $\forall f,g \in \Fspace$.
\end{packed_enum}
\end{theorem}
Appendix~\ref{app:state_learned} provides the proof.

Analysis of both the \dfslearn and \ondemand routines requires a
careful inductive argument.  We first consider the \dfslearn routine.
\begin{theorem}[Guarantee for \dfslearn]
\label{thm:dfs_learn}
Consider running \dfslearn on path $p$ with regressors $\Fcal$, and parameters $\epsilon,\delta$.
With probability at least $1-\delta$, for all $h$ and all $s_h \in \Sspace_h$ for which we called \regelim, the conclusions of Theorem~\ref{thm:td_regelim} hold with $\phi = \frac{\epsilon}{320H^2\sqrt{K}}$ and $\tau_1 = 20(H-h)\sqrt{K}\phi$. 
If $T$ is the number of times the algorithm calls \regelim, then the number of episodes executed by the algorithm is at most,
\begin{align*}
\order\left(\frac{TH^4\Asize^2}{\epsilon^2}\log(\Fsize \Ssize \Asize H/\delta)\right).
\end{align*}
Moreover, $T \le MH$ for any execution of \dfslearn.
\end{theorem}
The proof details are deferred to Appendix~\ref{app:dfs_learn}.

A simple consequence of Theorem~\ref{thm:dfs_learn} is that we can estimate $V^\star$ accurately once we have called \dfslearn on $\treeroot$. 
\begin{corollary}[Estimating $V^\star$]
\label{cor:dfs_learn}
Consider running \dfslearn at $\treeroot$ with regressors $\Fcal$, and parameters $\epsilon,\delta$.
Then with probability at least $1-\delta$, the estimate $\hat{V}^\star$ satisfies, 
\begin{align*}
|\hat{V}^\star - V^\star| \le \epsilon/8.
\end{align*}
Moreover the algorithm uses at most,
\begin{align*}
\order\left(\frac{\Ssize H^5 \Asize^2}{\epsilon^2}\log\left(\frac{\Fsize \Ssize H \Asize}{\delta}\right)\right)
\end{align*}
trajectories.
\end{corollary}
\begin{proof}
Since we ran \dfslearn at $\treeroot$, we may apply Theorem~\ref{thm:dfs_learn}.
By specification of the algorithm, we certainly ran \regelim at $\treeroot$, which is at level $h=1$, so we apply the conclusions in Theorem~\ref{thm:td_regelim}. 
In particular, we know that $f^\star \in \Fcal$ and that for any surviving $f \in \Fcal$,
\begin{align*}
|\hat{V}^f(p,\pi_f) - V^\star| &= |\hat{V}^f(p,\pi_f) - V^f(p,\pi_f) + V^f(p,\pi_f) - V^{f^\star}(p,\pi_{f^\star})|\\
& \le \frac{\phi}{\sqrt{12}} + 8\phi\sqrt{K} + 2\phi + 20(H-1)\sqrt{K}\phi \le \epsilon/8.
\end{align*}
The last bound follows from the setting of $\phi$ and $\tau_1$. 
Since our estimate $\hat{V}^\star$ is $\hat{V}^f(p,\pi_f)$ for some surviving $f$, we guarantee estimation error at most $\epsilon/8$. 

As for the sample complexity, Theorem~\ref{thm:dfs_learn} shows that the total number of executions of \regelim can be at most $MH$, which is our setting of $T$.
\end{proof}

Finally we turn to the \ondemand routine.
\begin{theorem}[Guarantee for \ondemand]
\label{thm:on_demand}
Consider running \ondemand with regressors $\Fcal$, estimate $\hat{V}^\star$ and parameters $\epsilon, \delta$ and assume that $|\hat{V}^\star - V^\star| \le \epsilon/8$. 
Then with probability at least $1-\delta$, \ondemand terminates after at most,
\begin{align*}
\otil\left(\frac{MH^6K^2}{\epsilon^3}\log(N/\delta)\log(1/\delta)\right)
\end{align*}
trajectories and it returns a policy $\pi_f$ with $V^\star - V(\treeroot, \pi_f) \le \epsilon$.
\end{theorem}
See Appendix~\ref{app:on_demand} for details.

\section{Proof of Theorem~\ref{thm:det_dfs_big_obs}}
The proof of the main theorem follows from straightforward application of Theorems~\ref{thm:dfs_learn} and~\ref{thm:on_demand}. 
First, since we run \dfslearn at the root, $\treeroot$, the bias and estimation bounds in Theorem~\ref{thm:td_regelim} apply at $\treeroot$, so we guarantee accurate estimation of the value $V^\star$ (See Corollary~\ref{cor:dfs_learn}). 
This is required by the \ondemand routine, but at this point, we can simply apply Theorem~\ref{thm:on_demand}, which is guaranteed to find a $\epsilon$-suboptimal policy and also terminate in $MH$ iterations. 
Combining these two results, appropriately allocating the failure probability $\delta$ evenly across the two calls, and accumulating the sample complexity bounds establishes Theorem~\ref{thm:det_dfs_big_obs}.

\section{Proof of Theorem~\ref{thm:td_regelim}}
\label{app:regelim}
The proof of Theorem~\ref{thm:td_regelim} is quite technical, and we
compartmentalize into several components. 
Throughout we will use the preconditions of the theorem, which we reproduce here.
\begin{condition}
\label{cond:estimation}
For all $f \in \Fspace$ and $a \in \Aspace$, we have estimates $\fpvhat{f}{p\circ a}$ such that,
\begin{align*}
|\fpvhat{f}{p\circ a} - \fpvfunc{f}{p\circ a}| \le \esterr.
\end{align*}
\end{condition}
\begin{condition}
\label{cond:bias}
For all $f,g \in \Fspace$ and $a \in \Aspace$ we have,
\begin{align*}
|\fpvfunc{f}{p \circ a} - \fpvfunc{g}{p\circ a}| \le \tdbias.
\end{align*}
\end{condition}
We will make frequent use of the parameters $\esterr$ and $\tdbias$ which are specified by these two conditions, and explicit in the theorem statement. 

Recall the notation,
\begin{align*}
\fgpvfunc{f}{p}{g} \triangleq \EE_{x\sim \dst{p}} f(x,\pi_g(x)),
\end{align*}
which will be used heavily throughout the proof.

We will suppress dependence on the distribution $\dst{p}$, since we
are considering one invocation of \regelim and we always roll into
$p$.  This means that all (observation, reward) tuples will be
drawn from $\dst{p}$.  Secondly it will be convenient to introduce the
shorthand $\fpvshort{f}{p} = \fpvfunc{f}{p}$ and similarly for the
estimates. Finally, we will further shorten the value functions for
paths $p\circ a$ by defining,
\begin{align*}
\favshort{f}{a} \triangleq \EE_{x\sim \dst{p\circ a}} f(x,\pi_f(x)) = \fpvfunc{f}{p\circ a}.
\end{align*}
We will also use $\favhatshort{f}{a}$ to denote the estimated versions which we have according to Condition~\ref{cond:estimation}.

Lastly, our proof makes extensive use of the following random
variable, which is defined for a particular regressor $f \in \Fspace$:
\begin{align*}
Y(f) \triangleq (f(x,a) - r(a) - \fpvhatshort{f}{p\circ a})^2 - (f^\star(x,a) - r(a) - \fpvhatshort{f^\star}{p\circ a})^2.
\end{align*}
Here $(x,r) \sim \dst{p}$ and $a \in \Aspace$ is drawn uniformly at
random as prescribed by Algorithm~\ref{alg:td_regressor_elimination}.
We use $Y(f)$ to denote the
random variable associated with regressor $f$, but sometimes drop the dependence on $f$ when it is clear from context.

To proceed, we first compute the expectation and variance of this
random variable.

\begin{lemma}[Properties of TD Squared Loss]
\label{lem:y_mean_var}
Assume Condition~\ref{cond:estimation} holds.
Then for any $f \in \Fspace$, the random variable $Y$ satisfies,
\begin{align*}
\EE_{x,a,r}[Y] &= \EE_{x,a}\left[ (f(x,a) - \fpvhatshort{f}{p\circ a}
  - f^\star(x,a) + \fpvshort{f^\star}{p\circ a})^2 \right] -
\EE_{x,a}\left[(\fpvhatshort{f^\star}{p\circ a} -
  \fpvshort{f^\star}{p\circ a})^2 \right]\\
\Var_{x,a,r}[Y] &\le 32\EE_{x,a}[Y] + 64\esterr^2.
\end{align*}
\end{lemma}

\begin{proof}
For shorthand, denote $f = f(x,a), f^\star = f^\star(x,a)$ and recall the definition of $\favshort{f}{a}$ and $\favhatshort{f}{a}$.
\begin{align*}
  & \EE_{x,a,r}Y\\
  & = \EE_{x,a,r}\left[ (f - \favhatshort{f}{a} - r(a))^2 - (f^\star - \favhatshort{f^\star}{a} - r(a))^2\right]\\
  & = \EE_{x,a,r}\left[ (f - \favhatshort{f}{a})^2 - 2r(a)(f - \favhatshort{f}{a} -
    f^\star + \favhatshort{f^\star}{a}) - (f^\star -
    \favhatshort{f^\star}{a})^2\right]
\end{align*}
Now recall that $\EE[r(a)|x,a] = f^*(x,a) - \favshort{f^\star}{a}$ by
definition of $f^*$, which allows us to deduce,
\begin{align*}
  & \EE_{x,a,r}Y\\
  & = \EE_{x,a}\left[ (f - \favhatshort{f}{a})^2 - 2(f^\star - \favshort{f^\star}{a})(f - \favhatshort{f}{a}) +2(f^\star - \favhatshort{f^\star}{a} + \favhatshort{f^\star}{a} - \favshort{f^\star}{a})(f^\star - \favhatshort{f^\star}{a}) - (f^\star - \favhatshort{f^\star}{a})^2\right]\\
  & = \EE_{x,a}\left[ (f - \favhatshort{f}{a})^2- 2(f^\star - \favshort{f^\star}{a})(f - \favhatshort{f}{a}) + (f^\star-\favhatshort{f^\star}{a})^2 + 2(\favhatshort{f^\star}{a} - \favshort{f^\star}{a})(f^\star - \favhatshort{f^\star}{a})\right]\\
  &= \EE_{x,a}\left[ (f - \favhatshort{f}{a})^2- 2(f^\star - \favshort{f^\star}{a})(f - \favhatshort{f}{a}) + (f^\star-\favshort{f^\star}{a} + \favshort{f^\star}{a} - \favhatshort{f^\star}{a})^2 + 2(\favhatshort{f^\star}{a} - \favshort{f^\star}{a})(f^\star - \favhatshort{f^\star}{a})\right]\\
  & = \EE_{x,a}\left[ (f - \favhatshort{f}{a} - f^\star + \favshort{f^\star}{a})^2 + 2(\favshort{f^\star}{a} - \favhatshort{f^\star}{a})(f^\star - \favshort{f^\star}{a}) + (\favshort{f^\star}{a} - \favhatshort{f^\star}{a})^2+ 2(\favhatshort{f^\star}{a} - \favshort{f^\star}{a})(f^\star - \favhatshort{f^\star}{a})\right]\\
  & = \EE_{x,a}\left[ (f - \favhatshort{f}{a} - f^\star + \favshort{f^\star}{a})^2 - (\favshort{f^\star}{a} - \favhatshort{f^\star}{a})^2\right].
\end{align*}
For the second claim, notice that we can write,
\begin{align*}
Y = (f - \favhatshort{f}{a} - f^\star + \favhatshort{f^\star}{a})(f - \favhatshort{f}{a} + f^\star - \favhatshort{f^\star}{a} - 2r(a)),
\end{align*}
so that,
\begin{align*}
Y^2 \le 16(f - \favhatshort{f}{a} - f^\star + \favhatshort{f^\star}{a})^2.
\end{align*}
This holds because all quantities in the second term are bounded in $[0,1]$.
Therefore,
\begin{align*}
\Var(Y) & \le \EE[Y^2] \\
& \le 16 \EE_{x,a}\left[(f(x,a) - \favhatshort{f}{a} - f^\star(x,a) + \favhatshort{f^\star}{a})^2\right]\\
& = 16 \EE_{x,a}\left[(f(x,a) - \favhatshort{f}{a} - f^\star(x,a) + \favshort{f^\star}{a} + \favhatshort{f^\star}{a} - \favshort{f^\star}{a})^2\right]\\
& \le 32 \EE_{x,a}\left[(f(x,a) - \favhatshort{f}{a} - f^\star(x,a) + \favshort{f^\star}{a})^2\right] + 32 \esterr^2\\
& \le 32 \EE_{x,a}Y + 64 \esterr^2
\end{align*}
The first inequality is straightforward, while the second inequality
is from the argument above.  The third inequality uses the fact that
$(a+b)^2 \le 2a^2 + 2b^2$ and the fact that for each $a$, the estimate
$\favhatshort{f^\star}{a}$ has absolute error at most $\esterr$ (By Condition~\ref{cond:estimation}).
The last inequality adds and subtracts the term involving
$(\favshort{f^\star}{a} - \favhatshort{f^\star}{a})^2$ to obtain $\EE_{x,a}Y$.
\end{proof}

The next step is to relate the empirical squared loss to the population squared loss, which is done by application of Bernstein's inequality.
\begin{lemma}[Squared Loss Deviation Bounds]
\label{lem:sq_loss_dev}
Assume Condition~\ref{cond:estimation} holds.
With probability at least $1-\delta/2$, where $\delta$ is a parameter of the algorithm, $f^\star$ survives the filtering step of Algorithm~\ref{alg:td_regressor_elimination} and moreover, any surviving $f$ satisfies,
\begin{align*}
\EE Y(f) \le 6\esterr^2 + \frac{120\log(2\Fsize/\delta)}{\ntrain}.
\end{align*}
\end{lemma}
\begin{proof}
We will apply Bernstein's inequality on the centered random variable,
\begin{align*}
\sum_{i=1}^{\ntrain} Y_i(f) - \EE Y_i(f),
\end{align*} 
and then take a union bound over all $f \in \Fspace$.
Here the expectation is over the $\ntrain$ samples $(x_i,a_i,r_i)$ where $(x_i,r) \sim \dst{p}$, $a_i$ is chosen uniformly at random, and $r_i = r(a_i)$.
Notice that since actions are chosen uniformly at random, all terms in the sum are identically distributed, so that $\EE Y_i(f) = \EE Y(f)$.

To that end, fix one $f \in \Fspace$ and notice that $|Y - \EE Y| \le 8$ almost surely, as each quantity in the definition of $Y$ is bounded in $[0,1]$, so each of the four terms can be at most $4$, but two are non-positive and two are non-negative in $Y - \EE Y$. 
We will use Lemma~\ref{lem:y_mean_var} to control the variance. 
Bernstein's inequality implies that, with probability at least $1-\delta$,
\begin{align*}
\sum_{i=1}^{\ntrain} \EE Y_i - Y_i &\le \sqrt{2 \sum_{i}\Var(Y_i)\log(1/\delta)} + \frac{16\log(1/\delta)}{3}\\
& \le \sqrt{64 \sum_i \left(\EE(Y_i) + 2\esterr^2\right)\log(1/\delta)} + \frac{16\log(1/\delta)}{3}
\end{align*}
The first inequality here is Bernstein's inequality while the second is based on the variance bound in Lemma~\ref{lem:y_mean_var}.

Now letting $X = \sqrt{\sum_i (\EE(Y_i) + 2\esterr^2)}$, $Z = \sum_i Y_i$ and $C = \sqrt{\log(1/\delta)}$, the inequality above is equivalent to,
\begin{align*}
& X^2 - 2\ntrain\esterr^2 - Z \le 8XC + \frac{16}{3}C^2\\
& \Rightarrow X^2 - 8XC + 16C^2 -Z \le 2\ntrain \esterr^2 + 22C^2\\
& \Rightarrow (X - 4C)^2 -Z \le 2\ntrain \esterr^2  + 22C^2\\
& \Rightarrow -Z \le 2\ntrain \esterr^2 + 22C^2.
\end{align*}
Using the definition of $-Z$, this last inequality implies
\begin{align*}
\sum_{i=1}^{\ntrain} (f^\star(x_i,a_i) - r_i(a_i) - \fpvhatshort{f^\star}{p\circ a_i})^2 \le \sum_{i=1}^{\ntrain} (f(x_i,a_i) - r_i(a_i) - \fpvhatshort{f}{p\circ a_i})^2 + 2\ntrain\esterr^2 + 22\log(1/\delta).
\end{align*}
Via a union bound over all $f \in \Fspace$, rebinding $\delta \gets \delta/(2\Fsize)$, and dividing through by $\ntrain$, we have,
\begin{align*}
\tilde{R}(f^\star) \le \min_{f \in \Fspace} \tilde{R}(f) + 2\esterr^2 + \frac{22\log(2\Fsize/\delta)}{\ntrain}.
\end{align*}
Since this is precisely the threshold used in filtering regressors, we ensure that $f^\star$ survives. 

Now for any surviving regressor $f$, we are ensured that $Z$ is upper
bounded in the elimination step~\eqref{eqn:elim}.  Specifically we
have,
\begin{align*}
(X - 4C)^2 &\le Z + 2\ntrain\esterr^2 + 22C^2 \le 4\ntrain\esterr^2 + 44C^2\\
\Rightarrow X^2 & \le (\sqrt{4\ntrain\esterr^2 + 44C^2} + 4C)^2\\
& \le 8\ntrain\esterr^2 + 120C^2.
\end{align*}
This proves the claim since $X^2 = \ntrain \EE Y(f) + 2\ntrain\esterr^2$ (Recall that the $Y_i$s are identically distributed). 
\end{proof}

This deviation bound allows us to establish the three claims in Theorem~\ref{thm:td_regelim}.
We start with the estimation error claim, which is straightforward.
\begin{lemma}[Estimation Error]
\label{lem:est_bd}
Let $\delta \in (0,1)$. Then with probability at least $1-\delta$, for all $f \in \Fspace$ that are retained by the Algorithm~\ref{alg:td_regressor_elimination}, we have estimates $\fpvhat{f}{p}$ with,
\begin{align*}
|\fpvhat{f}{p} - \fpvfunc{f}{p}| \le \sqrt{\frac{2\log(2\Fsize/\delta)}{\ntrain}}.
\end{align*}
\end{lemma}
\begin{proof}
The proof is a consequence of Hoeffding's inequality and a union bound. 
Clearly the Monte Carlo estimate,
\begin{align*}
\fpvhat{f}{p} = \frac{1}{\ntrain}\sum_{i=1}^{\ntrain}f(x_i,\pi_f(x_i)),
\end{align*}
is unbiased for $\fpvfunc{f}{p}$ and the centered quantity is bounded in $[-1,1]$. 
Thus Hoeffding's inequality gives precisely the bound in the lemma.
\end{proof}

Next we turn to the claim regarding bias.
\begin{lemma}[Bias Accumulation]
\label{lem:bias_bd}
Assume Conditions~\ref{cond:estimation} and~\ref{cond:bias} hold.
In the same $1-\delta/2$ event in Lemma~\ref{lem:sq_loss_dev}, for any pair $f,g \in \Fspace$ retained by Algorithm~\ref{alg:td_regressor_elimination}, we have,
\begin{align*}
\fpvfunc{f}{p} - \fpvfunc{g}{p} \le 2\sqrt{K}\sqrt{7\esterr^2 + \frac{120\log(2\Fsize/\delta)}{\ntrain}} + 2\esterr + \tdbias
\end{align*}
\end{lemma}
\begin{proof}
Throughout the proof, we use $\EE_x[\cdot]$ to denote expectation when $x \sim \dst{p}$. 
We start by expanding definitions,
\begin{align*}
\fpvfunc{f}{p} - \fpvfunc{g}{p} = \EE_{x} [f(x,\pi_f(x))
  - g(x,\pi_g(x))]
\end{align*}
Now, since $g$ prefers $\pi_g(x)$ to $\pi_f(x)$, it must be the case that $g(x,\pi_g(x)) \ge g(x,\pi_f(x))$, so that,
\begin{align*}
\fpvfunc{f}{p} - \fpvfunc{g}{p} \le & \hspace{0.4cm} \EE_{x} f(x,\pi_f(x)) - g(x,\pi_f(x))\\
= & \hspace{0.4cm} \EE_{x}[f(x,\pi_f(x)) - \fgpvhat{f}{p\circ \pi_f(x)}{f} - f^\star(x,\pi_f(x)) +\fgpvfunc{f^\star}{p\circ \pi_f(x)}{f^\star}]\\
& - \EE_{x}[g(x,\pi_f(x)) - \fgpvhat{g}{p\circ \pi_f(x)}{g} - f^\star(x,\pi_f(x)) +\fgpvfunc{f^\star}{p\circ \pi_f(x)}{f^\star}]\\
& + \EE_{x}[\fgpvhat{f}{p\circ \pi_f(x)}{f} - \fgpvhat{g}{p \circ \pi_f(x)}{g}].
\end{align*}
This last equality is just based on adding and subtracting terms. 
The first two terms look similar, and we will relate them to the squared loss. 
For the first, by Lemma~\ref{lem:y_mean_var}, we have that for each $x \in \Xspace$,
\begin{align*}
& \EE_{r,a|x}[Y(f)] + \EE_{a|x}[(\fpvhat{f^\star}{p\circ a} - \fpvfunc{f^\star}{p \circ a})^2]\\
& = \EE_{a|x}\left[(f(x,a) - \fpvhat{f}{p\circ a} - f^\star(x,a) + \fpvfunc{f^\star}{p \circ a})^2\right]\\
& \ge \frac{1}{\Asize}\left[(f(x,\pi_f(x)) - \fpvhat{f}{p\circ \pi_f(x)} - f^\star(x,\pi_f(x)) + \fpvfunc{f^\star}{p \circ \pi_f(x)})^2\right].
\end{align*}
The equality is Lemma~\ref{lem:y_mean_var} while the inequality follows from the fact that each action, in particular $\pi_f(x)$, is played with probability $1/K$ and the quantity inside the expectation is non-negative. 
Now by Jensen's inequality the first term can be upper bounded as,
\begin{align*}
& \EE_{x}[f(x,\pi_f(x)) - \fgpvhat{f}{p\circ \pi_f(x)}{f} - f^\star(x,\pi_f(x)) +\fgpvfunc{f^\star}{p\circ \pi_f(x)}{f^\star}]\\
& \le \sqrt{\EE_{x}[(f(x,\pi_f(x)) - \fgpvhat{f}{p\circ \pi_f(x)}{f} - f^\star(x,\pi_f(x)) +\fgpvfunc{f^\star}{p\circ \pi_f(x)}{f^\star})^2]}\\
& = \sqrt{K \EE_{x}\left[\frac{1}{K} (f(x,\pi_f(x)) - \fgpvhat{f}{p\circ \pi_f(x)}{f} - f^\star(x,\pi_f(x)) +\fgpvfunc{f^\star}{p\circ \pi_f(x)}{f^\star})^2\right]}\\
& \le \sqrt{K \left(\EE_{x,a,r}[Y(f)] + \EE_{x,a}[(\fpvhat{f^\star}{p\circ a} - \fpvfunc{f^\star}{p\circ a})^2]\right)}\\
& \le \sqrt{K} \sqrt{\EE Y(f) + \esterr^2}\\
& \le \sqrt{K} \sqrt{7 \esterr^2 + \frac{120\log(\Fsize/\delta)}{\ntrain}},
\end{align*}
where the last step follows from Lemma~\ref{lem:sq_loss_dev}. This
bounds the first term in the expansion of $\fpvfunc{f}{p} -
\fpvfunc{g}{p}$.  Now for the term involving $g$, we can apply
essentially the same argument,
\begin{align*}
& - \EE_{x}[g(x,\pi_f(x)) - \fgpvhat{g}{p\circ \pi_f(x)}{g} - f^\star(x,\pi_f(x)) +\fgpvfunc{f^\star}{p\circ \pi_f(x)}{f^\star}]\\
& \le \sqrt{\EE_{x}[(g(x,\pi_f(x)) - \fgpvhat{g}{p\circ \pi_f(x)}{g} - f^\star(x,\pi_f(x)) +\fgpvfunc{f^\star}{p\circ \pi_f(x)}{f^\star})^2]}\\
& \le \sqrt{K}\sqrt{7\esterr^2 + \frac{120\log(\Fsize/\delta)}{\ntrain}}
\end{align*}
Summarizing, the current bound we have is,
\begin{align}
  \fpvfunc{f}{p} - \fpvfunc{g}{p} \le 2\sqrt{K}\sqrt{7\esterr^2 + \frac{120\log(\Fsize/\delta)}{\ntrain}} + \EE_{x}[\fgpvhat{f}{p\circ \pi_f(x)}{f} - \fgpvhat{g}{p \circ \pi_f(x)}{g}]
  \label{eqn:disagree}
\end{align}
The last term is easily bounded by the preconditions in Theorem~\ref{thm:td_regelim}.
For each $a$, we have,
\begin{align*}
& \fgpvhat{f}{p\circ a}{f} - \fgpvhat{g}{p \circ a}{g}\\
& \le |\fgpvhat{f}{p\circ a}{f} -\fpvfunc{f}{p\circ a}| + |\fpvfunc{f}{p\circ a} - \fpvfunc{g}{p\circ a}| + |\fpvfunc{g}{p \circ a} - \fgpvhat{g}{p \circ a}{g}|\\
& \le 2\esterr + \tdbias,
\end{align*}
from Conditions~\ref{cond:estimation}
and~\ref{cond:bias}. Consequently,
\begin{align*}
& \EE_{x}[\fgpvhat{f}{p\circ \pi_f(x)}{f} - \fgpvhat{g}{p
      \circ \pi_f(x)}{g}] \\ & = \sum_{a \in \Aspace} \EE_x
  \left[\mathbf{1}[\pi_f(x) = a](\fgpvhat{f}{p\circ a}{f} -
    \fgpvhat{g}{p \circ a}{g})\right]\\ & \le 2\esterr + \tdbias.
\end{align*}
This proves the claim.
\end{proof}

Lastly, we must show how the squared loss relates to the risk, which helps establish the last claim of the theorem.
The proof is similar to that of the bias bound but has subtle differences that require reproducing the argument.
\begin{lemma}[Instantaneous Risk Bound]
\label{lem:regret_bd}
Assume Conditions~\ref{cond:estimation} and~\ref{cond:bias} hold.
In the same $1-\delta/2$ event in Lemma~\ref{lem:sq_loss_dev}, for any regressor $f \in \Fspace$ retained by Algorithm~\ref{alg:td_regressor_elimination}, we have,
\begin{align*}
\fgpvfunc{f^\star}{p}{f^\star} - \fgpvfunc{f^\star}{p}{f} \le \sqrt{2K}\sqrt{7\esterr^2 + \frac{120\log(2\Fsize/\delta)}{\ntrain}} + 2(\esterr+ \tdbias).
\end{align*}
\end{lemma}
\begin{proof}
\begin{align*}
\fpvfunc{f^\star}{p} - \fgpvfunc{f^\star}{p}{f} &= \EE_{x}[f^\star(x,\pi_{f^\star}(x)) - f^\star(x,\pi_f(x))]\\
& \le \EE_x [f^\star(x,\pi_{f^\star}(x)) - f(x,\pi_{f^\star}(x)) + f(x,\pi_f(x)) - f^\star(x,\pi_f(x))].
\end{align*}
This follows since $f$ prefers its own action to that of $f^\star$, so that $f(x,\pi_f(x)) \ge f(x,\pi_{f^\star}(x))$. 
For any observation $x \in \Xspace$ and action $a \in \Aspace$, define,
\begin{align*}
\Delta_{x,a} = (f(x,a) - \fpvhatshort{f}{p\circ a} - f^\star(x,a) + \fpvshort{f^\star}{p\circ a}),
\end{align*}
where $\fpvshort{f}{p} = \EE_{x \sim D_p}[f(x,\pi_f(x))]$ and
similarly for $\fpvhatshort{p}$.  Then we can write,
\begin{align*}
&\fpvfunc{f^\star}{p} - \fgpvfunc{f^\star}{p}{f}\\
& \le \EE_x [\Delta_{x,\pi_f(x)} - \Delta_{x,\pi_{f^\star}(x)} + \fpvhatshort{f}{p\circ \pi_f(x)} - \fpvshort{f^\star}{p\circ \pi_f(x)} - \fpvhatshort{f}{p\circ\pi_{f^\star}(x)} + \fpvshort{f^\star}{p\circ \pi_{f^\star}(x)}].
\end{align*}
The term involving both $\Delta$s can be bounded as in the proof of Lemma~\ref{lem:bias_bd}.
For any $x \in \Xspace$
\begin{align*}
&\EE_{r,a|x}Y(f) + \EE_{a|x}[(\fpvhatshort{f^\star}{p\circ a} - \fpvshort{f^\star}{p\circ a})^2]\\
& = \EE_{a|x}\left[(f(x,a) - \fpvhatshort{f}{p \circ a} - f^\star(x,a) + \fpvshort{f^\star}{p\circ a})^2\right]\\
& \ge \frac{\Delta_{x,\pi_f(x)}^2 + \Delta_{x,\pi_{f^\star}(x)}^2}{K} \ge \frac{(\Delta_{x,\pi_{f^\star}(x)} - \Delta_{x,\pi_f(x)})^2}{2K}.
\end{align*}
Thus,
\begin{align*}
\EE_x [\Delta_{x,\pi_f(x)} - \Delta_{x,\pi_{f^\star}(x)}] &\le \sqrt{2K \EE \frac{(\Delta_{x,\pi_f(x)} - \Delta_{x,\pi_{f^\star}(x)})^2}{2K}}\\
& \le \sqrt{2K}\sqrt{\EE Y(f) + \esterr^2} \le \sqrt{2K}\sqrt{7\esterr^2 + \frac{120\log(2\Fsize/\delta)}{\ntrain}}.
\end{align*}
We are left to bound the residual term,
\begin{align*}
& (\fpvhatshort{f}{p\circ \pi_f(x)} - \fpvshort{f^\star}{p\circ \pi_f(x)} - \fpvhatshort{f}{p\circ\pi_{f^\star}(x)} + \fpvshort{f^\star}{p\circ \pi_{f^\star}(x)})\\
& \le \left|\fpvshort{f}{p\circ \pi_f(x)} - \fpvshort{f^\star}{p\circ \pi_f(x)} - \fpvshort{f}{p\circ\pi_{f^\star}(x)} + \fpvshort{f^\star}{p\circ \pi_{f^\star}(x)}\right| + 2\esterr\\
& \le 2(\esterr + \tdbias).
\end{align*}
\end{proof}

Notice that Lemma~\ref{lem:regret_bd} above controls the quantity $V^{f^\star}(p,\pi_{f^\star}) - V^{f^\star}(p,\pi_f)$ which is the difference in values of the optimal behavior from $p$ and the policy that first acts according to $\pi_f$ and then behaves optimally thereafter.
This is \emph{not} the same as acting according to $\pi_f$ for all subsequent actions. 
We will control this cumulative risk $V^\star(p) - V(p,\pi_f)$ in the second phase of the algorithm.

\textbf{Proof of Theorem~\ref{thm:td_regelim}:}
Equipped with the above lemmas, we can proceed to prove the theorem. 
By assumption of the theorem, Conditions~\ref{cond:estimation} and~\ref{cond:bias} hold, so all lemmas are applicable. 
Apply Lemma~\ref{lem:est_bd} with failure probability $\delta/2$, where $\delta$ is the parameter in the algorithm, and apply Lemma~\ref{lem:sq_loss_dev}, which also fails with probability at most $\delta/2$. 
A union bound over these two events implies that the failure probability of the algorithm is at most $\delta$. 

Outside of this failure event, all three of Lemmas~\ref{lem:est_bd},~\ref{lem:bias_bd}, and~\ref{lem:regret_bd} hold. 
If we set $\ntrain = 24\log(4N/\delta)/\esterr^2$ then these four bounds give,
\begin{align*}
|\fpvhat{f}{p} - \fpvfunc{f}{p}| &\le \frac{\esterr}{\sqrt{12}}\\
|\fpvfunc{f}{p} - \fpvfunc{g}{p}| & \le 8\esterr\sqrt{K} + 2\esterr + \tdbias\\
\fgpvfunc{f^\star}{p}{f^\star} - \fgpvfunc{f^\star}{p}{f} & \le 4\esterr\sqrt{2K} + 2\esterr + 2\tdbias.
\end{align*}
These bounds hold for all $f,g \in \Fspace$ that are retained by the algorithm.
Of course by Lemma~\ref{lem:sq_loss_dev}, we are also ensured that $f^\star$ is retained by the algorithm.



\section{Proof of Theorem~\ref{thm:state_learned}}
\label{app:state_learned}
This result is a straightforward application of Hoeffding's inequality. 
We collect $\ntest$ observations $x_i \sim \dst{p}$ by applying path $p$ from the root and use the Monte Carlo estimates,
\begin{align*}
\fpvhat{f}{p} = \frac{1}{\ntest}\sum_{i=1}^{\ntest}f(x_i,\pi_f(x_i)).
\end{align*}
By Hoeffding's inequality, via a union bound over all $f \in \Fspace$, we have that with probability at least $1-\delta$,
\begin{align*}
\left|\fpvhat{f}{p} - \fpvfunc{f}{p}\right| \le \sqrt{\frac{2 \log(2\Fsize/\delta)}{\ntest}}.
\end{align*}
Setting $\ntest = 2\log(2\Fsize/\delta)/\esterr^2$, gives that our empirical estimates are at most $\esterr$ away from the population versions. 

Now for the first claim, if the population versions are already within $\slbias$ of each other, then the empirical versions are at most $2\esterr+\slbias$ apart by the triangle inequality,
\begin{align*}
|\fpvhat{f}{p} - \fpvhat{g}{p}| & \le |\fpvhat{f}{p} - \fpvfunc{f}{p}| + |\fpvfunc{f}{p} - \fpvfunc{g}{p}| + |\fpvfunc{g}{p} - \fpvhat{g}{p}|\\
& \le 2\esterr + \slbias.
\end{align*}
This applies for any pair $f,g \in \Fspace$ whose population value predictions are within $\slbias$ of each other.
Since we set $\epstest \ge 2\esterr + \slbias$ in Theorem~\ref{thm:state_learned}, this implies that the procedure returns \true. 

For the second claim, if the procedure returns \true, then all empirical value predictions are at most $\epstest$ apart, so the population versions are at most $2\esterr + \epstest$ apart, again by the triangle inequality. 
Specifically, for any pair $f,g \in \Fspace$ we have,
\begin{align*}
|\fpvfunc{f}{p} - \fpvfunc{g}{p}| & \le |\fpvfunc{f}{p} - \fpvhat{f}{p}| + |\fpvhat{f}{p} - \fpvhat{g}{p}| + |\fpvhat{g}{p} - \fpvfunc{g}{p}|\\
& \le 2\esterr + \epstest.
\end{align*}
Both arguments apply for all pairs $f,g \in \Fspace$, which proves the claim. 

\section{Proof of Theorem~\ref{thm:dfs_learn}}
\label{app:dfs_learn}
Assume that all calls to \regelim and \statelearned operate successfully, i.e., we can apply Theorems~\ref{thm:td_regelim} and~\ref{thm:state_learned} on any path $p$ for which the appropriate subroutine has been invoked.
We will bound the number of calls and hence the total failure probability. 

Recall that $\epsilon$ is the error parameter passed to \dfslearn and that we set $\esterr = \frac{\epsilon}{320 H^2\sqrt{K}}$. 

We first argue that in all calls to \regelim, the estimation precondition is satisfied. 
To see this, notice that by design, the algorithm only calls \regelim at path $p$ after the recursive step, which means that for each $a$, we either ran \regelim on $p\circ a$ or \statelearned returned \true on $p\circ a$. 
Since both Theorems~\ref{thm:td_regelim} and~\ref{thm:state_learned} guarantee estimation error of order $\esterr$, the estimation precondition for path $p$ holds. 
This argument applies to all paths $p$ for which we call \regelim, so that the estimation precondition is always satisfied. 

We next analyze the bias term, for which proceed by induction. 
To state the inductive claim, we define the notion of an \emph{accessed path}.
We say that a path $p$ is \emph{accessed} if either (a) we called \regelim on path $p$ or (b) we called \statelearned on $p$ and it returned \true.

The induction is on the number of actions remaining, which we denote with $\eta$.
At time point $h$ there are $H-h+1$ actions remaining.

\textbf{Inductive Claim:} For all accessed paths $p$ with $\eta$ actions remaining and any pair $f,g \in \Fcal$ of surviving regressors,
\begin{align*}
|V^f(p,\pi_f) - V^g(p,\pi_g)| \le 20\eta\sqrt{K}\esterr.
\end{align*}

\textbf{Base Case:} The claim clearly holds when $\eta=0$ since there are zero actions remaining and all regressors estimate future reward as zero.

\textbf{Inductive Step:} Assume that the inductive claim holds for all accessed paths with $\eta-1$ actions remaining. 
Consider any accessed path $p$ with $\eta$ actions remaining. 
Since we access the path $p$, either we call \regelim or \statelearned returns \true.
If we call \regelim, then we access the paths $p\circ a$ for all $a \in \Acal$.
By the inductive hypothesis, we have already filtered the regressor class so that for all $a \in \Acal, f,g \in \Fcal$, we have,
\begin{align*}
|V^f(p \circ a,\pi_f) - V^g(p\circ a, \pi_f)| \le 20(\eta-1)\sqrt{K} \esterr.
\end{align*}
We instantiate $\tau_1 = 20(\eta-1)\sqrt{K}\esterr$ in the bias precondition of Theorem~\ref{thm:td_regelim}. 
We also know that the estimation precondition is satisfied with parameter $\esterr$.
The bias bound of Theorem~\ref{thm:td_regelim} shows that, for all $f,g \in \Fcal$ retained by the algorithm,
\begin{align}
|V^f(p,\pi_f) - V^g(p,\pi_g)| &\le 8\esterr\sqrt{K} + 2\esterr + \tau_1 \notag\\
& \le 10\esterr\sqrt{K} + 20(\eta-1)\esterr\sqrt{K} \le 20(\eta-\frac{1}{2})\esterr\sqrt{K}. \label{eqn:bias-induct}
\end{align}
Thus, the inductive step holds in this case. 

The other case we must consider is if \statelearned returns \true.
Notice that for a path $p$ with $\eta$ actions to go, we call
\statelearned with parameter $\epstest = 20(\eta-1/4)\sqrt{K}\esterr$.
We actually invoke the routine on path $p$ when we are currently
processing a path $p'$ with $\eta+1$ actions to go (i.e., $p = p'\circ a$
for some $a \in \Acal$), so we set $\epstest$ in terms of $H-|p'|-5/4
= \eta-1/4$. ($|p|$ is actually one less than the level of the state
reached by applying $p$ from the root.)  Then, by
Theorem~\ref{thm:state_learned}, we have the bias bound,
\begin{align*}
|V^f(p,\pi_f) - V^g(p,\pi_f)| & \le 2\esterr + 20(\eta-1/4)\sqrt{K}\esterr\\
& \le 20\eta\sqrt{K} \esterr.
\end{align*}
Thus, we have established the inductive claim. 

\textbf{Verifying preconditions for Theorem~\ref{thm:td_regelim}:}
To apply the conclusions of Theorem~\ref{thm:td_regelim} at some state $s$, we must verify that the preconditions hold, with the appropriate parameter settings, before we execute \regelim.
We saw above that the estimation precondition always holds with parameter $\esterr$, assuming successful execution of all subroutines. 
The inductive argument also shows that the bias precondition also holds with $\tau_1 = 20(\eta-1)\sqrt{K}\phi$ for a state $s \in \Scal_{H-\eta+1}$ that we called \regelim on. 
Thus, both preconditions are satisfied at each execution of \regelim, so the conclusions of Theorem~\ref{thm:td_regelim} apply at any state $s$ for which we have executed the subroutine.
Note that the precondition parameters that we use here, specifically $\tau_1$, depend on the actions-to-go $\eta$.

Substituting the level $h$ for the actions-to-go $\eta$ gives $\tau_1 = 20(H-h)\sqrt{K}\phi$ at level $h$.


\textbf{Sample Complexity:}
We now bound the number of calls to each subroutine, which reveals how to allocate the failure probability and gives the sample complexity bound. 
Again assume that all calls succeed. 

First notice that if we call \statelearned on some state $s$ with $\eta$ actions-to-go for which we have already called \regelim, then \statelearned returns \true~(assuming all calls to subroutines succeed).
This follows because \regelim guarantees that the population predicted values are at most $20(\eta-1/2)\sqrt{K}\esterr$ apart (Eq.~\eqref{eqn:bias-induct}), which becomes the choice of $\slbias$ in application of Theorem~\ref{thm:state_learned}. 
This is valid since,
\begin{align*}
2\esterr + 20(\eta-1/2)\sqrt{K}\esterr \leq 20(\eta-1/4)\sqrt{K}\esterr = \epstest,
\end{align*}
so that the precondition for Theorem~\ref{thm:state_learned} holds. 
Thus, at any level $h$, we can call \regelim~at most one time per state $s \in \Sspace_h$. 
In total, this yields $\Ssize H$ calls to \regelim.

Next, since we only make recursive calls when we execute $\regelim$, we expand at most $\Ssize$ paths per level.
This means that we call \statelearned~on at most $\Ssize\Asize$ paths per level, since the fan-out of the tree is $\Asize$.
Thus, the number of calls to \statelearned~is at most $\Ssize\Asize H$. 

By our setting $\delta$ in the subroutine calls (i.e. $\delta/(2\Ssize\Asize H)$ in calls to \statelearned~and $\delta/(2\Ssize H)$ in calls to \regelim), and by Theorems~\ref{thm:td_regelim} and~\ref{thm:state_learned}, the total failure probability is therefore at most $\delta$. 

Each execution of \regelim requires $\ntrain$ trajectories while executions of \statelearned require $\ntest$ trajectories. 
Since before each execution of \regelim we always perform $K$ executions of \statelearned, if we perform $T$ executions of \regelim, the total sample complexity is bounded by,
\begin{align*}
T(\ntrain + K\ntest) &\le (3\times 10^6)\frac{TH^4K}{\epsilon^2}\log(8NMH/\delta) + (3\times 10^5) \frac{TH^4K^2}{\epsilon^2}\log(4NMKH/\delta)\\
= \ & \order\left(\frac{TH^4\Asize^2}{\epsilon^2}\log\left(\frac{\Fsize \Ssize H\Asize }{\delta}\right)\right).
\end{align*}
The total number of executions of \regelim can be no more than $MH$, by the argument above.


\section{Analysis for \ondemand}
\label{app:on_demand}

Throughout the proof, assume that $|\hat{V}^\star - V^\star| \le \epsilon/8$.
We will ensure that the first half of the algorithm guarantees this.
Let $\Ecal$ denote the event that all Monte-Carlo estimates $\hat{V}(\treeroot,\pi_f)$ are accurate and all calls to \dfslearn succeed (so that we may apply Theorem~\ref{thm:dfs_learn}).
By accurate, we mean,
\begin{align*}
|\hat{V}(\treeroot,\pi_f) - V(\treeroot,\pi_f)| \le \epsilon/8. 
\end{align*}
Formally, $\Ecal$ is the intersection over all executions of \dfslearn of the event that the conclusions of Theorem~\ref{thm:dfs_learn} apply for this execution and the intersection over all iterations of the loop in \ondemand of the event that the Monte Carlo estimate $\hat{V}(\treeroot,\pi_f)$ is within $\epsilon/8$ of $V(\treeroot,\pi_f)$.
We will bound this failure probability, i.e. $\PP[\bar{\Ecal}]$, toward the end of the proof. 

\begin{lemma}[Risk bound upon termination]
\label{lem:risk_upon_termination}
If $\Ecal$ holds, then when \ondemand terminates, it outputs a policy $\pi_f$ with $V^\star - V(\pi_f) \le \epsilon$.
\end{lemma}
\begin{proof}
The proof is straightforward.
\begin{align*}
V^\star - V(\pi_f) &\le |V^\star - \hat{V}^\star| + |\hat{V}^\star - \hat{V}(\pi_f)| + |\hat{V}(\pi_f) - V(\pi_f)|\\
& \le \epsilon/8 + \epsilon/2 + \epsilon/8 = 3\epsilon/4 \le \epsilon.
\end{align*}
The first bound follows by assumption on $\hat{V}^\star$ while the second comes from the definition of $\epsdemand$ and the third holds under event $\Ecal$.
\end{proof}

\begin{lemma}[Termination Guarantee]
\label{lem:termination_bd}
If $\Ecal$ holds, then when \ondemand selects a policy that is at most $\epsilon/4$-suboptimal, it terminates.
\end{lemma}
\begin{proof}
We must show that the test succeeds, for which we will apply the triangle inequality,
\begin{align*}
|\hat{V}^\star - \hat{V}(\pi_f)| &\le |\hat{V}^\star - V^\star| + |V^\star - V(\pi_f)| + |V(\pi_f) - \hat{V}(\pi_f)|\\
& \le \epsilon/8 + \epsilon/4 + \epsilon/8 \le \epsilon/2 = \epsdemand.
\end{align*}
Therefore the test is guaranteed to succeed.
Again the last bound here holds under event $\Ecal$. 
\end{proof}

At some point in the execution of the algorithm, define a set of
\emph{learned states} $L$ as
\begin{align}
L(\Fcal) \triangleq \bigcup_{h} \left\{s \in \Scal_h: \max_{f \in \Fcal}V^\star(s)- V^{f^\star}(s,\pi_f) \le 4\esterr\sqrt{2K} + 2\esterr + 40(H-h)\sqrt{K}\phi\right\}.\label{eq:l_def}
\end{align}
By Theorem~\ref{thm:td_regelim}, any state for which we have
successfully called \regelim is $L(\Fcal)$, since the condition is
precisely the instantaneous risk bound.  Since we only ever call
\regelim through \dfslearn, the fact that these calls to \regelim
succeeded is implied by the event $\Ecal$. The \emph{unlearned states}
are denoted $\bar{L}$, where the dependence on $\Fcal$ is left
implicit.


For a policy $\pi_f$, let $q^{\pi_f}[s \rightarrow \bar{L}]$ denote
the probability that when behaving according to $\pi_f$ starting from
state $s$, we visit an unlearned state.  We now show that
$q^{\pi_f}[\treeroot \rightarrow \bar{L}]$ is related to the risk of
the policy $\pi_f$.


\begin{lemma}[Policy Risk]
\label{lem:regret_with_unlearned}
Define $L$ as in Eq.~\eqref{eq:l_def} and define $q^{\pi_f}[s \rightarrow \bar{L}]$ accordingly. 
Assume that $\Ecal$ holds and let $f$ be a surviving regressor, so that $\pi_f$ is a surviving policy.
Then,
\begin{align*}
V^\star - V(\treeroot,\pi_f) \le q^{\pi_f}[\treeroot \rightarrow \bar{L}] + 40\sqrt{K}\esterr H^2.
\end{align*}
\end{lemma}
\begin{proof}
Recall that under event $\Ecal$, we can apply the conclusions of Theorem~\ref{thm:td_regelim} with $\phi = \frac{\epsilon}{320H^2\sqrt{K}}$ and $\tau_1 = 20(H-h)\sqrt{K}\phi$ for any $h$ and state $s \in \Scal_h$ for which we have called \regelim. 
Our proof proceeds by creating a recurrence relation through application of Theorem~\ref{thm:td_regelim} and then solving the relation.
Specifically, we want to prove the following inductive claim.

\noindent\textbf{Inductive Claim:} For a state $s \in L$ with $\eta$ actions to go,
\begin{align*}
V^\star(s) - V(s,\pi_f) \le 40\esterr\sqrt{K}\eta^2 + q^{\pi_f}[s\rightarrow \bar{L}].
\end{align*}

\noindent\textbf{Base Case:} With zero actions to go, all policies achieve zero reward and no policies visit $\bar{L}$ from this point, so the inductive claim trivially holds.

\noindent \textbf{Inductive Step:} For the inductive hypothesis, consider some state $s$ at level $h$, for which \regelim has successfully been called. 
There are $\eta = H-h+1$ actions to go.
By Theorem~\ref{thm:dfs_learn}, we know that,
\begin{align*}
V^\star(s) - V^{f^\star}(s,\pi_f) \le 4\esterr\sqrt{2K} + 2\esterr + 2\tau_1,
\end{align*}
with $\tau_1 = 20(H-h) \esterr\sqrt{K}$. This bound is clearly at most $40\eta\esterr\sqrt{K}$. 
Now,
\begin{align*}
V^\star(s) - V(s,\pi_f) &= V^\star(s) - V^{f^\star}(s,\pi_f) + V^{f^\star}(s,\pi_f) - V(s,\pi_f)\\
& \le 40\eta \esterr\sqrt{K} + \EE_{(x,r) \sim D_s} r(\pi_f(x)) + V^\star(s\circ \pi_f(x)) - r(\pi_f(x)) - V(s\circ \pi_f(x), \pi_f).
\end{align*}
Let us focus on just the second term, which is equal to,
\begin{align*}
& \EE_{x \sim D_s}\left[ \left(V^\star(s\circ \pi_f(x)) - V(s\circ \pi_f(x), \pi_f)\right)(\mathbf{1}[\Gamma(s,\pi_f(x)) \in L] + \mathbf{1}[\Gamma(s,\pi_f(x)) \notin L])\right]\\
& \le \sum_{s' \in L}\PP_{x \sim D_s}[\Gamma(s,\pi_f(x)) = s']\left(V^\star(s') - V(s', \pi_f)\right) + \PP_{x \sim D_s}[\Gamma(s,\pi_f(x)) \notin L].
\end{align*}
Since all of the recursive terms above correspond only to states $s' \in L$, we may apply the inductive hypothesis, to obtain the bound,
\begin{align*}
& 40\eta\esterr\sqrt{K} + \sum_{s' \in L} \PP_{x \in D_s}[\Gamma(s,\pi_f(x)) = s']\left( 40(h-1)^2\esterr\sqrt{K} + q^{\pi_f}[s'\rightarrow \bar{L}]\right) + \PP_{x \sim D_s}[\Gamma(s,\pi_f(x)) \notin L]\\
& \le 40\eta\esterr\sqrt{K} + 40(\eta-1)^2\esterr\sqrt{K} + q^{\pi_f}[s\rightarrow \bar{L}]\\
& \le 40\esterr\sqrt{K}\eta^2 + q^{\pi_f}[s\rightarrow \bar{L}].
\end{align*}
Thus, we have proved the inductive claim.
Applying at the root of the tree gives the result. 
\end{proof}

Recall that we set $\esterr = \frac{\epsilon}{320 H^2\sqrt{K}}$ in \dfslearn.
This ensures that $40 H^2\esterr\sqrt{K} \le \epsilon/8$, which means that if $q^{\pi_f}[\treeroot \rightarrow \bar{L}] = 0$, then we ensure $V^\star - V(\treeroot,\pi_f) \le \epsilon/8$. 

\begin{lemma}[Each non-terminal iteration makes progress]
\label{lem:on_demand_progress}
Assume that $\Ecal$ holds. 
If $\pi_f$ is selected but fails the test, then with probability at least $1-\exp(-\epsilon\ndemandtwo/8)$, at least one of the $\ndemandtwo$ trajectories collected visits a state $s \notin L$.
\end{lemma}
\begin{proof}
First, if $\pi_f$ fails the test, we know that,
\begin{align*}
\epsdemand < |\hat{V}(\treeroot,\pi_f) - \hat{V}^\star| \le \epsilon/4 + |V(\treeroot,\pi_f) - V^\star|,
\end{align*}
which implies that,
\begin{align*}
\epsilon/4 < V^\star - V(\treeroot,\pi_f).
\end{align*}
On the other hand Lemma~\ref{lem:regret_with_unlearned}, shows that,
\begin{align*}
V^\star - V(\treeroot,\pi_f) \le q^{\pi_f}[\treeroot \rightarrow \bar{L}] + 40H^2\sqrt{K} \esterr.
\end{align*}
Using our setting of $\esterr$, and combining the two bounds gives,
\begin{align*}
\epsilon/4 < q^{\pi_f}[\treeroot \rightarrow \bar{L}] + \epsilon/8
\Rightarrow q^{\pi_f}[\treeroot \rightarrow \bar{L}] > \epsilon/8.
\end{align*}
Thus, the probability that all $\ndemandtwo$ trajectories miss $\bar{L}$ is,
\begin{align*}
\PP[\textrm{all trajectories miss } \bar{L}] &= (1-q^{\pi_f}[\treeroot \rightarrow \bar{L}])^{\ndemandtwo}\\
& \le (1-\epsilon/8)^{\ndemandtwo} \le \exp(-\epsilon \ndemandtwo/8).
\end{align*}
Therefore, we must hit $\bar{L}$ with substantial probability. 
\end{proof}

\subsection{Proof of Theorem~\ref{thm:on_demand}}
Again assume that $\Ecal$ holds. 
First, by Lemma~\ref{lem:risk_upon_termination}, we argued that if \ondemand terminates, then it outputs a policy that satisfies the PAC-guarantee. 
Moreover, by Lemma~\ref{lem:termination_bd}, we also argued that if \ondemand selects a policy that is at most $\epsilon/4$ suboptimal, then it terminates.
Thus the goal of the proof is to show that it quickly finds a policy that is at most $\epsilon/4$ suboptimal. 

Every execution of the loop in \ondemand either passes the test or fails the test at level $\epsdemand$.
If the test succeeds, then Lemma~\ref{lem:risk_upon_termination} certifies that we have found an $\epsilon$-suboptimal policy, thus establishing the PAC-guarantee.
If the test fails, then Lemma~\ref{lem:on_demand_progress} guarantees that we call \dfslearn on a state that was not previously trained on.
Thus at each non-terminal iteration of the loop, we call \dfslearn and hence \regelim on at least one state $s \notin L$, so that the set of learned states grows by at least one.
By Lemma~\ref{lem:regret_with_unlearned} and our setting of $\esterr$, if we have called \regelim on all states at all levels, then we guarantee that all surviving policies have risk at most $\epsilon/8$. 
Thus the number of iterations of the loop is at most $MH$ since that is the number of unique states in the model.

\textbf{Bounding $\PP[\bar{\Ecal}]$:} Since we have bounded the total
number of iterations, we are now in a position to assign failure
probabilities and bound the event $\Ecal$.  Actually we must consider
not only the event $\Ecal$ but also the event that all non-terminal
iterations visit some state $s \notin L$.  Call this new event
$\Ecal'$ which is the intersection of $\Ecal$ with the event that all
unsuccessful iterations visit $\bar{L}$.

More formally, we use the fact that for events $A_0, \ldots, A_t$, we have,
\begin{align}
\PP[\bigcup_{i=0}^tA_i] \le \PP[A_0] + \sum_{i=1}^t \PP[A_i | \bar{A}_0, \ldots, \bar{A}_{i-1}].
\label{eq:union_bd}
\end{align}
This inequality is based on applying the union bound to the events $A'_i = (A_i \cap \bigcap_{j=0}^{i-1} A_j)$. 

Our analysis above bounds events of this form, namely the probability of a failure event conditioned on no previous failure event occurring.
Specifically, we decompose $\Ecal'$ into three types of events. 
\begin{enumerate}
\item $B_t^{(1)}$ denotes the event that the Monte Carlo estimate $\hat{V}(\treeroot,\pi_f)$ is accurate for the $t^{\textrm{th}}$ iteration of the while loop.
\item $B_t^{(2)}$ denotes the event that \dfslearn succeeds at the $t^{\textrm{th}}$ iteration of the while loop.
\item $B_t^{(3)}$ denotes the event that $t$ is a non-terminal iteration and we visit $\bar{L}$ at the $t^{\textrm{th}}$ iteration. 
\end{enumerate}
These events are defined for $t \in [MH]$, since we know that if all
events hold we will perform at most $MH$ iterations.
$\Ecal'$ is the intersection of all of these events. 

The failure probability can be expressed as,
\begin{align*}
\PP[\bar{\Ecal}'] &= \PP[\bigcup_{t=1}^{MH} \bar{B}_t^{(1)} \cup \bar{B}_t^{(2)} \cup \bar{B}_t^{(3)}],
\end{align*}
and via Equation~\ref{eq:union_bd}, it suffices to bound each event, conditioned on all previous success events.

We have $\delta$ probability to allocate, and since we perform at most $MH$ iterations, we allocate $\delta/(MH)$ probability to each iteration and $1/3$ of the available failure probability to each type of event.

For the initial Monte-Carlo estimate in event $B_t^{(1)}$, by Hoeffding's inequality, we know that,
\begin{align*}
|\hat{V}(\treeroot,\pi_f) - V(\treeroot,\pi_f)| \le \sqrt{\frac{\log(6MH/\delta)}{2\ndemand}}.
\end{align*}
We want this bound to be at most $\epsilon/8$ which requires,
\begin{align*}
\ndemand \ge \frac{32\log(6MH/\delta)}{\epsilon^2}.
\end{align*}
This bound holds for any fixed $\pi_f$, and it is independent of previous events.

For the second event, for each of the $H\ndemandtwo$ calls to \dfslearn, we set the parameter to be $\delta/(3MH^2\ndemandtwo)$, so that by Theorem~\ref{thm:dfs_learn}, we may apply Theorem~\ref{thm:td_regelim} at all states that we have called \regelim on.
Again this bounds the probability of $\bar{B}_t^{(2)}$, independently of previous events. 

Finally, conditioned on $B_{t}^{(1)}$, we may apply Lemma~\ref{lem:on_demand_progress} at iteration $t$ to observe that the the conditional probability of $\bar{B}_t^{(3)}$ is at most $\exp(-\ndemandtwo \epsilon/8)$.
And for this to be smaller than $\delta/(3MH)$ we require,
\begin{align*}
\ndemandtwo \ge \frac{8\log(3MH/\delta)}{\epsilon}.
\end{align*}
Both conditions on $\ndemand$ and $\ndemandtwo$ are met by our choices in the algorithm specification.

In total, if we set, $\ndemand = \frac{32\log(6MH/\delta)}{\epsilon^2}$ and $\ndemandtwo = 8\log(3MH/\delta)/\epsilon$ in \ondemand and if \ondemand always call \dfslearn with parameter $\delta/(3MH^2\ndemandtwo)$ we guarantee that the total failure probability for this subroutine is at most $\delta$. 

\textbf{Sample Complexity:}
It remains to bound the sample complexity for the execution of \ondemand.
We do at most $MH$ iterations, and in each iteration we use $\ndemand$ trajectories to compute Monte-Carlo estimates, contributing an $MH\ndemand$ to the sample complexity.
We also call \dfslearn on each of the $H\ndemandtwo$ prefixes collected during each iteration so that there are at most $MH^2\ndemandtwo$ calls to \dfslearn in total. 
Na\"{i}vely, each call to \dfslearn takes at most $O(\frac{MH^5K^2}{\epsilon^2}\log(\ndemandtwo NMKH/\delta))$ episodes, leading to a crude sample complexity bound of,
\begin{align*}
\otil\left(\frac{M^2H^7K^2}{\epsilon^3}\log(N/\delta)\log(1/\delta)\right).
\end{align*}
Recall that the $\otil$ notation suppresses all logarithmic factors except those involving $N$ and $\delta$. 

This bound can be significantly improved using a more careful argument.
Apart from the first call to \regelim in each application of \dfslearn, the total number of additional calls to \regelim is bounded by $MH$ since once we call \regelim on a state, \statelearned always returns \true. 

Each call to \regelim requires $\ntrain + K\ntest$ samples (because we always call \statelearned on all direct descendants before), and the total number of calls is at most,
\begin{align*}
MH^2\ndemandtwo + MH = \order\left(\frac{MH^2}{\epsilon}\log(MH/\delta)\right).
\end{align*}
With our settings of $\ntrain$ and $\ntest$, the sample complexity is therefore at most,
\begin{align*}
& \order\left(\frac{MH^6K^2}{\epsilon^3}\log(MHKN/(\epsilon\delta))\log(MH/\delta)\right)\\
& = \otil\left(\frac{MH^6K^2}{\epsilon^3}\log(N/\delta)\log(1/\delta)\right).
\end{align*}
This concludes the proof of Theorem~\ref{thm:on_demand}.

%% file: arxiv.bbl
\begin{thebibliography}{28}
\providecommand{\natexlab}[1]{#1}
\providecommand{\url}[1]{\texttt{#1}}
\expandafter\ifx\csname urlstyle\endcsname\relax
  \providecommand{\doi}[1]{doi: #1}\else
  \providecommand{\doi}{doi: \begingroup \urlstyle{rm}\Url}\fi

\bibitem[Agarwal et~al.(2012)Agarwal, Dud{\'\i}k, Kale, Langford, and
  Schapire]{agarwal2012contextual}
Alekh Agarwal, Miroslav Dud{\'\i}k, Satyen Kale, John Langford, and Robert~E
  Schapire.
\newblock Contextual bandit learning with predictable rewards.
\newblock In \emph{AISTATS}, 2012.

\bibitem[Antos et~al.(2008)Antos, Szepesv{\'a}ri, and Munos]{antos2008learning}
Andr{\'a}s Antos, Csaba Szepesv{\'a}ri, and R{\'e}mi Munos.
\newblock Learning near-optimal policies with {B}ellman-residual minimization
  based fitted policy iteration and a single sample path.
\newblock \emph{MLJ}, 2008.

\bibitem[Auer et~al.(2002)Auer, Cesa-Bianchi, Freund, and
  Schapire]{auer2002nonstochastic}
Peter Auer, Nicolo Cesa-Bianchi, Yoav Freund, and Robert~E Schapire.
\newblock The nonstochastic multiarmed bandit problem.
\newblock \emph{SICOMP}, 2002.

\bibitem[Azizzadenesheli et~al.(2016)Azizzadenesheli, Lazaric, and
  Anandkumar]{azizzadenesheli2016reinforcement}
Kamyar Azizzadenesheli, Alessandro Lazaric, and Animashree Anandkumar.
\newblock Reinforcement learning of {POMDP}s using spectral methods.
\newblock In \emph{COLT}, 2016.

\bibitem[Baird(1995)]{baird1995residual}
Leemon Baird.
\newblock Residual algorithms: Reinforcement learning with function
  approximation.
\newblock In \emph{ICML}, 1995.

\bibitem[Brafman and Tennenholtz(2003)]{brafman2003r}
Ronen~I Brafman and Moshe Tennenholtz.
\newblock R-max -- a general polynomial time algorithm for near-optimal
  reinforcement learning.
\newblock \emph{JMLR}, 2003.

\bibitem[Dann and Brunskill(2015)]{dann2015sample}
Christoph Dann and Emma Brunskill.
\newblock Sample complexity of episodic fixed-horizon reinforcement learning.
\newblock In \emph{NIPS}, 2015.

\bibitem[Dudik et~al.(2011)Dudik, Hsu, Kale, Karampatziakis, Langford, Reyzin,
  and Zhang]{dudik2011efficient}
Miroslav Dudik, Daniel Hsu, Satyen Kale, Nikos Karampatziakis, John Langford,
  Lev Reyzin, and Tong Zhang.
\newblock Efficient optimal learning for contextual bandits.
\newblock In \emph{UAI}, 2011.

\bibitem[Jiang et~al.(2015)Jiang, Kulesza, and Singh]{jiang2015abstraction}
Nan Jiang, Alex Kulesza, and Satinder Singh.
\newblock Abstraction selection in model-based reinforcement learning.
\newblock In \emph{ICML}, 2015.

\bibitem[Jong and Stone(2007)]{jong2007model}
Nicholas~K Jong and Peter Stone.
\newblock Model-based exploration in continuous state spaces.
\newblock In \emph{Abstraction, Reformulation, and Approximation}, 2007.

\bibitem[Kakade and Langford(2002)]{kakade2002approximately}
Sham Kakade and John Langford.
\newblock Approximately optimal approximate reinforcement learning.
\newblock In \emph{ICML}, 2002.

\bibitem[Kakade et~al.(2003)Kakade, Kearns, and
  Langford]{kakade2003exploration}
Sham Kakade, Michael~J Kearns, and John Langford.
\newblock Exploration in metric state spaces.
\newblock In \emph{ICML}, 2003.

\bibitem[Kearns and Singh(2002)]{kearns2002near}
Michael Kearns and Satinder Singh.
\newblock Near-optimal reinforcement learning in polynomial time.
\newblock \emph{MLJ}, 2002.

\bibitem[Kearns et~al.(1999)Kearns, Mansour, and Ng]{kearns1999approximate}
Michael~J Kearns, Yishay Mansour, and Andrew~Y Ng.
\newblock Approximate planning in large {POMDP}s via reusable trajectories.
\newblock In \emph{NIPS}, 1999.

\bibitem[Kearns et~al.(2002)Kearns, Mansour, and Ng]{kearns2002sparse}
Michael~J. Kearns, Yishay Mansour, and Andrew~Y. Ng.
\newblock A sparse sampling algorithm for near-optimal planning in large markov
  decision processes.
\newblock \emph{MLJ}, 2002.

\bibitem[Langford and Zhang(2008)]{langford2008epoch}
John Langford and Tong Zhang.
\newblock The epoch-greedy algorithm for multi-armed bandits with side
  information.
\newblock In \emph{NIPS}, 2008.

\bibitem[Li and Littman(2010)]{li2010reducing}
Lihong Li and Michael~L Littman.
\newblock Reducing reinforcement learning to {KWIK} online regression.
\newblock \emph{Ann. Math AI}, 2010.

\bibitem[Li et~al.(2006)Li, Walsh, and Littman]{li2006towards}
Lihong Li, Thomas~J Walsh, and Michael~L Littman.
\newblock Towards a unified theory of state abstraction for {MDP}s.
\newblock In \emph{ISAIM}, 2006.

\bibitem[Mansour(1999)]{mansour1999reinforcement}
Yishay Mansour.
\newblock Reinforcement learning and mistake bounded algorithms.
\newblock In \emph{COLT}, 1999.

\bibitem[Meuleau et~al.(1999)Meuleau, Peshkin, Kim, and
  Kaelbling]{meuleau1999learning}
Nicolas Meuleau, Leonid Peshkin, Kee-Eung Kim, and Leslie~Pack Kaelbling.
\newblock Learning finite-state controllers for partially observable
  environments.
\newblock In \emph{UAI}, 1999.

\bibitem[Mnih et~al.(2015)Mnih, Kavukcuoglu, Silver, Rusu, Veness, Bellemare,
  Graves, Riedmiller, Fidjeland, Ostrovski, Petersen, Charles, Amir,
  Antonoglou, King, Kumaran, Wierstra, Legg, and Hassabis]{mnih2015human}
Volodymyr Mnih, Koray Kavukcuoglu, David Silver, Andrei~A Rusu, Joel Veness,
  Marc~G Bellemare, Alex Graves, Martin Riedmiller, Andreas~K Fidjeland, Georg
  Ostrovski, Stig Petersen, Beattie Charles, Sadik Amir, Ioannis Antonoglou,
  Helen King, Dharshan Kumaran, Daan Wierstra, Shane Legg, and Demis Hassabis.
\newblock Human-level control through deep reinforcement learning.
\newblock \emph{Nature}, 2015.

\bibitem[Nguyen et~al.(2013)Nguyen, Maillard, Ryabko, and
  Ortner]{nguyen2013competing}
Phuong Nguyen, Odalric-Ambrym Maillard, Daniil Ryabko, and Ronald Ortner.
\newblock Competing with an infinite set of models in reinforcement learning.
\newblock In \emph{AISTATS}, 2013.

\bibitem[Pazis and Parr(2016)]{pazis2016efficient}
Jason Pazis and Ronald Parr.
\newblock Efficient {PAC}-optimal exploration in concurrent, continuous state
  {MDP}s with delayed updates.
\newblock In \emph{AAAI}, 2016.

\bibitem[Perkins and Precup(2002)]{perkins2002convergent}
Theodore~J Perkins and Doina Precup.
\newblock A convergent form of approximate policy iteration.
\newblock In \emph{NIPS}, 2002.

\bibitem[Reveliotis and Bountourelis(2007)]{reveliotis2007efficient}
Spyros Reveliotis and Theologos Bountourelis.
\newblock Efficient {PAC} learning for episodic tasks with acyclic state
  spaces.
\newblock \emph{DEDS}, 2007.

\bibitem[Strehl et~al.(2006)Strehl, Li, Wiewiora, Langford, and
  Littman]{strehl2006pac}
Alexander~L Strehl, Lihong Li, Eric Wiewiora, John Langford, and Michael~L
  Littman.
\newblock {PAC} model-free reinforcement learning.
\newblock In \emph{ICML}, 2006.

\bibitem[Sutton et~al.(1999)Sutton, McAllester, Singh, and
  Mansour]{sutton1999policy}
Richard~S Sutton, David~A McAllester, Satinder~P Singh, and Yishay Mansour.
\newblock Policy gradient methods for reinforcement learning with function
  approximation.
\newblock In \emph{NIPS}, 1999.

\bibitem[Tsitsiklis and Van~Roy(1997)]{tsitsiklis1997analysis}
John~N Tsitsiklis and Benjamin Van~Roy.
\newblock An analysis of temporal-difference learning with function
  approximation.
\newblock \emph{IEEE TAC}, 1997.

\end{thebibliography}
